\documentclass[twoside,11pt]{article}
\usepackage{jair, theapa, rawfonts}
\usepackage{helvet}
\usepackage{courier}
\usepackage[hyphens]{url}
\usepackage{graphicx}
\usepackage{amsthm}
\usepackage{amssymb}
\usepackage{amsmath}
\usepackage{algorithm}
\usepackage{algorithmic}
\usepackage{multirow}
\usepackage{caption}

\jairheading{81}{2024}{481-509}{05/24}{11/24}
\ShortHeadings{Robust Offline-to-Online RL}
{Wen, Yu, Yang, Chen, Bai \& Wang}
\firstpageno{481}

\usepackage{xcolor}  
\usepackage{nicefrac}

\usepackage{amsthm}
\newtheorem*{theorem*}{Theorem}
\newtheorem{theorem}{Theorem}
\newtheorem{lemma}{Lemma}

\newtheorem*{corollary*}{Corollary}

\def\cS{{\mathcal{S}}}
\def\cD{{\mathcal{D}}}
\def\cA{{\mathcal{A}}}

\def\cT{{\mathcal{T}}}

\def\VV{{\mathbb{V}}}
\def\EE{{\mathbb{E}}}

\def\hw{{\widehat{w}}}

\newtheorem{definition}{Definition}
\begin{document}

\title{Towards Robust Offline-to-Online Reinforcement Learning via Uncertainty and Smoothness}

\author{\name Xiaoyu Wen
        \email wenxiaoyu@mail.nwpu.edu.cn \\ 
        \addr Northwestern Polytechnical University \\ Xi'an, Shaanxi, China
        \AND
        \name Xudong Yu \email hit20byu@gmail.com \\
        \addr               Harbin Institute of Technology\\ Harbin, Heilongjiang, China
        \AND
        \name Rui Yang \email ryangam@connect.ust.hk \\
        \addr The Hong Kong University of Science and Technology\\ HongKong, China
        \AND
        \name Haoyuan Chen \email chen\_hy@mail.nwpu.edu.cn \\
        \addr Northwestern Polytechnical University \\ Xi'an, Shaanxi, China
        \AND
        \name Chenjia Bai
        \email baichenjia@pjlab.org.cn \\
        \emph{(Corresponding author)} \\
        \addr 
              Shanghai Artificial Intelligence Laboratory\\ Shanghai, China \\
              Shenzhen Research Institute of Northwestern Polytechnical University\\ Shenzhen, GuangDong, China
        \AND
        \name Zhen Wang
        \email w-zhen@nwpu.edu.cn \\
        \emph{(Corresponding author)} \\
        \addr
              Northwestern Polytechnical University \\ Xi'an, Shaanxi, China}


\maketitle

\begin{abstract}
To obtain a near-optimal policy with fewer interactions in Reinforcement Learning (RL), a promising approach involves the combination of offline RL, which enhances sample efficiency by leveraging offline datasets, and online RL, which explores informative transitions by interacting with the environment. Offline-to-Online RL provides a paradigm for improving an offline-trained agent within limited online interactions. However, due to the significant distribution shift between online experiences and offline data, most offline RL algorithms suffer from performance drops and fail to achieve stable policy improvement in offline-to-online adaptation. To address this problem, we propose the Robust Offline-to-Online (RO2O) algorithm, designed to enhance offline policies through uncertainty and smoothness, and to mitigate the performance drop in online adaptation. Specifically, RO2O incorporates Q-ensemble for uncertainty penalty and adversarial samples for policy and value smoothness, which enable RO2O to maintain a consistent learning procedure in online adaptation without requiring special changes to the learning objective. Theoretical analyses in linear MDPs demonstrate that the uncertainty and smoothness lead to tighter optimality bound in offline-to-online against distribution shift. Experimental results illustrate the superiority of RO2O in facilitating stable offline-to-online learning and achieving significant improvement with limited online interactions.
\end{abstract}

\section{Introduction}

Reinforcement learning (RL) has demonstrated remarkable success in tackling complex tasks, such as playing games \shortcite{rainbow,alphazero,dota} and controlling robots \shortcite{TRPO,ppo,sac} in recent years. Nonetheless, persistent critiques point to its limited adaptability in real-world scenarios. The efficacy of RL critically hinges upon access to an unbiased interactive environment and millions of unrestricted trial-and-error attempts \shortcite{natureDQN}. However, domains such as healthcare \shortcite{healthcare} and autonomous driving \shortcite{autodriving} often present challenges in online data collection due to safety, feasibility, and financial reasons.

Offline RL presents a distinctive advantage over online RL, as it enables the learning of policies directly from a fixed dataset collected by a behavior policy \shortcite{BatchRL,BCQ,BARC}. These datasets can be sourced from historical logs, demonstrations, or expert knowledge, furnishing valuable information to facilitate learning without the need for costly online data collection. However, the performance of current offline RL methods heavily relies on the coverage of the state-action space and the quality of stored trajectories \shortcite{dataset}. Furthermore, the lack of exploration hampers the agent's ability to discover superior policies \shortcite{challenges}. To address this issue, numerous studies focus on enhancing pre-trained offline agents through limited online interactions, known as Offline-to-Online RL \shortcite{awac,balancedreplay,iql}. This paradigm aims to rectify estimation bias, which remains unaddressed during offline training, and leads to further policy improvement through several online fine-tuning steps.

Despite the potential to integrate offline datasets and online experiences to optimize the agent, existing offline-to-online learning methods suffer from performance drops and struggle to efficiently improve policies, which hinders their applicability in real-world scenarios. At the initial stage of online fine-tuning, the agent's performance may heavily decline due to the distributional shift between offline datasets and online transitions \shortcite{awac,jsrl}. Moreover, the inclusion of low-quality data can have detrimental effects on performance and lead to skewed optimization. Prior efforts to address this issue involve altering the policy extraction procedure \shortcite{awac,iql}, incorporating behavior cloning regularization \shortcite{adaptivebc}, modifying data sampling methods \shortcite{balancedreplay,MOOSE}, or proposing policy expansion sets \shortcite{pex}. While these methods have made progress in mitigating performance drops, they still suffer form limited performance improvement due to the lack of effective mechanisms to enhance performance during the fine-tuning phase.

In this paper, we propose the Robust Offline-to-Online (RO2O) algorithm for RL, designed to address the distribution shift in the offline-to-online process and achieve efficient policy improvement during the fine-tuning phase. To achieve this, RO2O utilizes $Q$-ensembles to learn robust value functions, resulting in no performance drop during the initial stage of online fine-tuning. Additionally, RO2O incorporates the smoothness regularization of policies and value functions on out-of-distribution (OOD) states and actions ensuring robust performance even when the interacting trajectories in the training buffer deviate significantly from offline policies. Notably, RO2O offers the advantage of not requiring the transformation of the learning algorithm \shortcite{e2o} or policy composition \shortcite{pex} throughout the process. From a theoretical perspective, we prove that under the linear MDP assumption, the uncertainty and smoothness lead to a tighter optimality bound in offline-to-online against distribution shift. Empirical results showcase the favorable performance of RO2O during both offline pre-training and online fine-tuning. Compared to baseline algorithms, RO2O achieves efficient policy improvement without the need for specific explorations or modifications to the learning architecture. The code is available in this repository (\url{https://github.com/BattleWen/RO2O}).

\section{Related Work}

\paragraph{Offline-to-Online RL} A key challenge in offline-to-online process is the performance drop experienced at the initial stage, attributed to the distributional shift between offline data and online experiences. Previous approaches have attempted to address this issue by altering policy extraction \shortcite{awac,iql}, adjusting sampling methods \shortcite{balancedreplay}, expanding policy sets \shortcite{pex}, and modifying $Q$-function learning targets \shortcite{cal-ql}. However, these methods cannot consistently achieve effective policy improvement within the limited fine-tuning steps. Recently, ensembles have been incorporated for both pessimistic learning during offline training and optimistic exploration during online learning \shortcite{e2o}. While such an ensemble method improves the Offline-to-Online performance, it requires careful modifications of learning objectives when transferring the policy from offline to online.
In contrast, our work handles offline training and online fine-tuning in a consistent manner without algorithmic modifications. The proposed approach not only achieves better offline performance but also enables efficient policy improvement during online fine-tuning.


\paragraph{Ensembles in RL} Ensemble methods in RL have emerged as a powerful approach to improve the stability and performance of learning algorithms. In online RL, ensembles are utilized to capture epistemic uncertainty and improve exploration \shortcite{bootstrappedDQN,ucbexploration}. Recent methods also employ ensembles to mitigate estimation bias during Bellman updates \shortcite{td3,maxminQ} or enhance sample efficiency \shortcite{redq}. In the context of offline RL, ensembles are employed in both model-free methods \shortcite{pbrl,SAC-N} and model-based methods \shortcite{mopo,MOOSE} to characterize the uncertainty of $Q$-values or dynamics models. Notably, several works \shortcite{SAC-N,msg} estimate lower confidence bounds of $Q$-functions using ensembles, where EDAC \shortcite{SAC-N} primarily focuses on improving sample efficiency with gradient diversity and MSG \shortcite{msg} mainly emphasizes the importance of ensemble independence for effectively estimating uncertainty. Our approach extends upon these methodologies by incorporating a perturbed sample set, which differs from solely estimating uncertainty within the existing state-action space of the dataset. We primarily use ensembles to penalize the $Q$-values of the OOD samples and apply smooth regularization to ensure that the policies and $Q$-values of in-sample data and perturbed samples do not deviate too much. In this way, we can smooth them out within a small range beyond the dataset’s state space, resulting in more robust estimates.

\paragraph{Robustness in RL} Robustness has gained paramount importance in RL to ensure the reliability and stability of RL agents in diverse and challenging environments. In online RL, previous research has explored techniques such as domain randomization \shortcite{domainrandom}, policy smoothing \shortcite{sr2l}, and data augmentation methods \shortcite{s4rl} to improve performance.  Recently, an offline RL algorithm \shortcite{rorl} incorporates policy and value smoothing for OOD states, highlighting the significance of robustness in offline RL agents. These approaches typically focus on enhancing robustness against adversarial perturbations on observations or actions and validate their effectiveness through the synthesis of noisy data. In contrast, our focus is on the robustness of models to handle the distributional shift specifically in the Offline-to-Online RL setting.

\section{Preliminaries}

\paragraph{Offline-to-Online RL} The RL problem is typically formulated as Markov Decision Process (MDP), represented by the tuple $\mathcal{M}=(\mathcal{S}, \mathcal{A}, P, R,\gamma)$. In this framework, the agent's decision-making process is guided by a policy denoted as $\pi$, which maps environmental states $s \in \mathcal{S}$ to actions $a \in \mathcal{A}$. The agent's objective is to find an optimal policy, denoted as $\pi^*$, that maximizes the expected cumulative reward over time. For a policy $\pi$, the state-action value function, denoted as $Q^\pi(s, a)$, represents the expected cumulative reward starting from state $s$, taking action $a$, and following policy $\pi$ thereafter. The learning target for the value function in online RL, also referred to as the Bellman operator, can be expressed as: $$\mathcal{T}Q(s,a)=r(s,a) + \gamma \mathbb{E}_{s'\sim P(\cdot|s,a), a'\in \pi(\cdot|s')} Q(s',a').$$

In offline RL, learning is performed using a fixed dataset $\mathcal{D}=\{s_i, a_i, r_i, s'_i\}_{i=1}^n$ of historical interactions sampled from a behavior policy $\mu$. A key challenge in offline RL is the bootstrapped error caused by the distributional shift between behavior policies and learned policies. To mitigate the distributional shift, previous methods \shortcite{uncertainty,pbrl,rorl} leverage $Q$-ensembles to capture epistemic uncertainty and penalize $Q$-values with large uncertainties. When we estimate the empirical expectation from the dataset $\cD$, the Bellman operator becomes 
$\widehat{\cT}Q(s,a)=r(s,a) + \gamma \widehat{\EE}_{s'\sim P(\cdot|s,a), a'\sim \pi(\cdot|s')}(Q(s',a')-\alpha U(s',a'))$, where $U(s',a')$ denotes the estimated uncertainties, and $\alpha$ is used to adjust the degree of pessimism. Additionally, RORL \shortcite{rorl} employs smooth regularization on the policy and the value function for states near the dataset.  

Despite the advantage of leveraging large-scale offline data, the performance of pre-trained agents is often limited by the optimality and coverage of the datasets. Overestimation of value functions cannot be substantially corrected without interactions with the environment. To address this limitation, our work focuses on offline-to-online learning, aiming to improve agents by incorporating limited online interactions.

\paragraph{Linear MDPs}
Our theoretical derivations build on top of linear MDP assumptions. Least Squares Value Iteration (LSVI) \shortcite{lsvi-2020} is a classic method frequently used in the linear MDPs to calculate the closed-form solution. In linear MDPs \shortcite{lsvi-2020}, the transition dynamics and reward function take the following form, as
\begin{equation}
\nonumber
    \mathbb{P}_t(s_{t+1} \,|\, s_t, a_t) = \langle \varphi(s_{t+1}), \phi(s_t, a_t) \rangle, 
    \quad r(s_t, a_t) = \upsilon^\top \phi(s_t, a_t), \quad \forall(s_{t+1}, a_t, s_t)\in\mathcal{S}\times\mathcal{A}\times\mathcal{S},
\end{equation}
where the feature embedding $\phi: \mathcal{S}\times\mathcal{A}\mapsto \mathbb{R}^d$ is known and $\varphi$ is an unknown measures over $\mathcal{S}$. We further assume that the reward function $r:\mathcal{S}\times\mathcal{A}\mapsto[0, 1]$ is bounded, the feature is bounded by $\|\phi\|_2 \leq 1$ and $v$ is an unknown vector. We consider the settings of $\gamma=1$ in the following. Then for any policy $\pi$, the state-action value function is also linear to $\phi$, as
\begin{equation}
\nonumber
Q^{\pi}(s_t,a_t)=w^{\top} \phi(s_t, a_t).
\end{equation}
Given data $\mathcal{D}_m=\{s_t^i,a_t^i,r_t^i,s_{t+1}^i\}_{i\in[m]}$, the parameter of the $w$ can be solved via LSVI algorithm, as
\begin{equation}
\label{eq::appendix_OOD_LSVI}
\widehat w_t = \min_{w\in \mathbb{R}^d} \sum^m_{i = 1}\bigl(\phi(s^i_t, a^i_t)^\top w - r(s^i_t, a^i_t)- V_{t+1}(s^i_{t+1})\bigr)^2 
\end{equation}
where $V_{t+1}$ is the estimated value function in the $(t+1)$-th step. Following LSVI, the explicit solution to Equation \eqref{eq::appendix_OOD_LSVI} takes the form of
\begin{equation}
\nonumber
    \widehat w_t = \Lambda^{-1}_t\sum^m_{i = 1}\phi(s^i_t, a^i_t)y_t^i,
    \qquad {\rm where~}\Lambda_t = \sum^m_{i=1}\phi(s^i_t, a^i_t)\phi(s^i_t, a^i_t)^\top
\end{equation}
is the feature covariance matrix of the state-action pairs in the offline dataset, and $y_t^i=r(s^i_t, a^i_t)+ V_{t+1}(s^i_{t+1})$ is the Bellman target in regression.


\section{Methodology}

In this section, we present our methodology for addressing the challenges posed by the Offline-to-Online RL setting. The most significant challenge in this context is effectively transferring knowledge from the static dataset to cope with distributional shift in the dynamic online environment. 
To tackle this crucial issue, we propose the RO2O algorithm, a novel approach that combines $Q$-ensembles and robustness regularization. We begin by providing a motivating example to illustrate that current methods struggle to handle a large distributional shift effectively. Subsequently, we introduce our algorithm, which maintains a consistent architecture in both offline and online learning phases. Furthermore, we establish theoretical support for our approach.

\subsection{Motivating Example}
Offline-to-Online RL methods face challenges arising from distribution shift not only between learned policies and behavior policies but also between offline data and online transitions during the fine-tuning process. Robust performance is expected from offline algorithms despite the presence of online trajectories that deviate from the learned offline policies. To investigate this, we evaluate two state-of-the-art offline methods, i.e., CQL and IQL, with distinct distribution shifts to simulate the offline-to-online process. Specifically, we pre-train the agents using the halfcheetah-expert dataset from D4RL \shortcite{d4rl} benchmark, and inject synthetic distributional shift similar to that in online fine-tuning process to assess their robustness. The synthetic distributional shift is incorporated by adding samples from a different offline dataset, such as the halfcheetah-medium dataset. As depicted in the left panel of Figure~\ref{fig:motivation}, noticeable discrepancies in the trajectory distribution exist between the two datasets, indicating the presence of distributional shift.
 
We compare the performance of CQL, IQL, and our method to figure out whether the state-of-the-art methods can handle the synthetic distributional shift during fine-tuning. The experimental results are shown in the right panel of Figure~\ref{fig:motivation}. Our findings reveal that all methods experience a performance drop at the initial stage due to the significant distributional shift. However, in comparison to CQL and IQL, our method exhibits a milder degradation in performance. Moreover, CQL and IQL fail to recover from the deviation during the fine-tuning phase with a new dataset. This  inability is attributed to the presence of samples that deviate significantly from the region covered by the current policies, which affects the learning of policies and value functions. In contrast, our method demonstrates superior robustness, enabling effective policy improvement even in the presence of significant distributional shifts. As previously mentioned, it is expected to correct estimation bias and improve pre-trained policies within limited online interactions, while traditional methods struggle to accomplish this.

\begin{figure}[!t]
\centering
\includegraphics[width=0.33\textwidth]{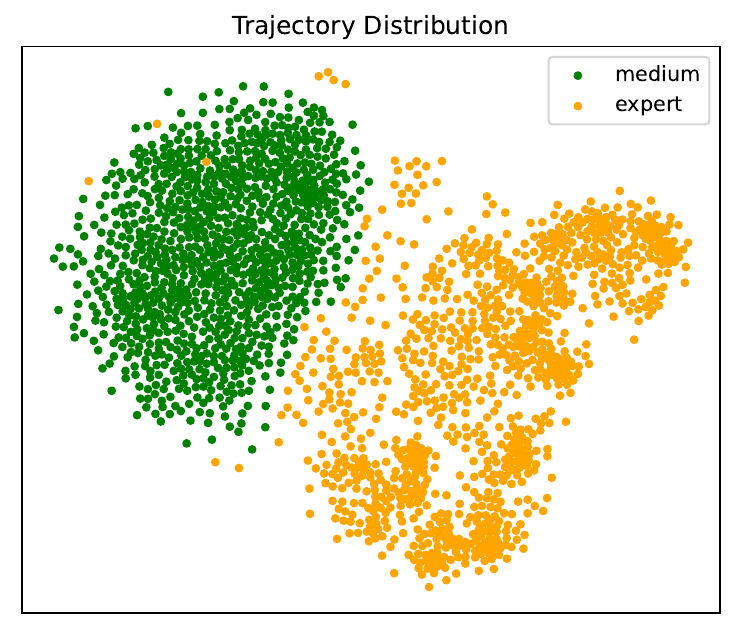}
\includegraphics[width=0.33\textwidth]{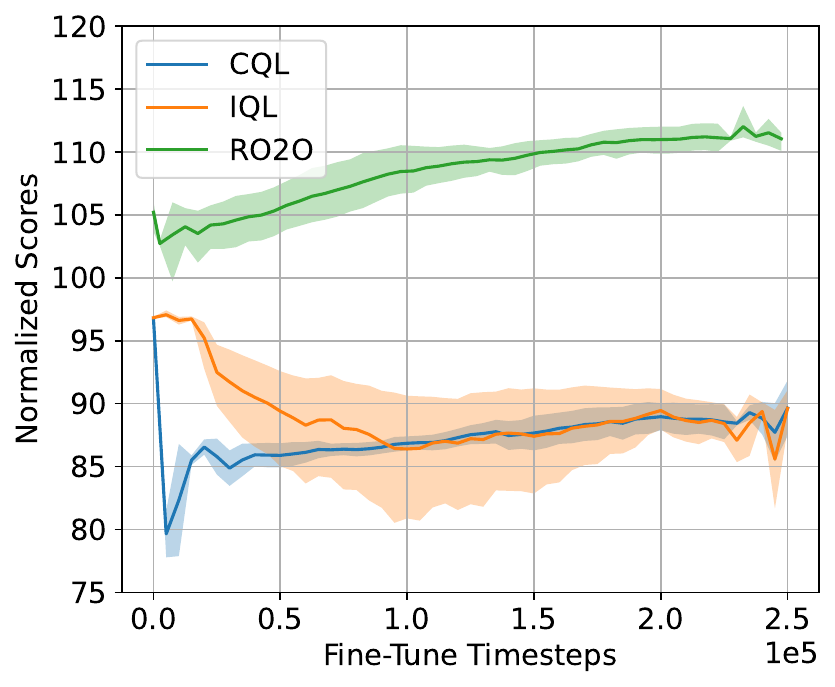}
\caption{Illustration for the motivating example. In the left panel, we visualize the trajectory distribution of two datasets, by mapping the trajectories into two-dimensional points using T-SNE \shortcite{tsne}. The right panel presents the fine-tuning performance.
}
\label{fig:motivation}
\end{figure}

\subsection{Algorithm}

\begin{figure*}[ht]
\centering
\includegraphics[width=\textwidth]{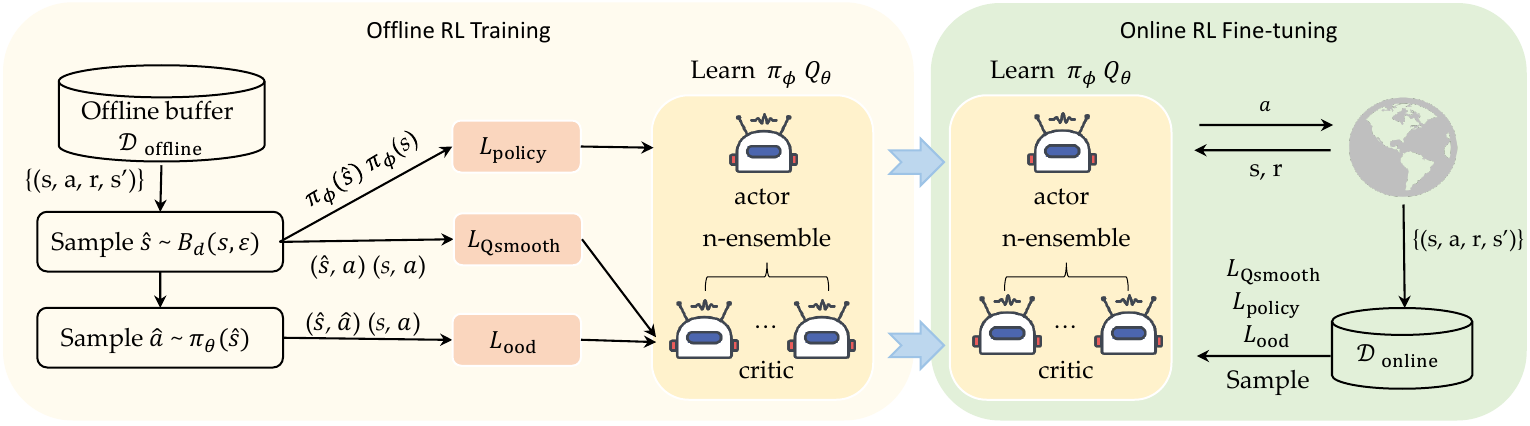}
\caption{\textbf{Overall framework of RO2O.} RO2O employs the same off-policy RL algorithms during the offline-to-online training phase. By using OOD sampling, we incorporate $\mathcal{L}_{\rm{ood}}$ and $\mathcal{L}_{\rm{Qsmooth}}$ into the training process for the gradient update, while also calculating $\mathcal{L}_{\text {policy}}$ to constrain the policy $\pi_{\zeta}(\hat{s})$ as close as possible to the current policy $\pi_{\zeta}(s)$.}
\label{fig-framework}
\end{figure*}

Based on the motivating example, we suggest that it is important to design a robust algorithm capable of ensuring stable policy improvement with online interactions. To this end, we propose the RO2O method, which incorporates ensembles and robustness regularization into the offline-to-online learning process. Notably, our method stands out from other existing approaches, such as PEX \shortcite{pex} and E2O \shortcite{e2o}, due to its unique characteristics. Unlike these methods, our approach does not require any changes to the learning algorithm or the need to conduct policy expansion when transitioning to the fine-tuning phase.

\subsubsection{Ensemble-based Learning} 
In our method, we adopt $Q$-ensemble with $N$ networks in both offline pre-training and online fine-tuning, employing the same update procedure. These ensemble networks possess identical architecture and are initialized independently. While prior studies \shortcite{redq,e2o} suggest that randomly selecting two of the $N$ ensembles is effective during the online learning phase, our findings indicate that choosing the \emph{minimum} of ensemble $Q$-functions is sufficient to achieve favorable performance. Moreover, this choice remains consistent with the offline learning process, where pessimism is necessary to counteract overestimation bias. Formally, the TD target when using the minimum of ensemble $Q$-functions can be expressed as:
\begin{equation}
\widehat{\cT} Q_{\theta_i}(s,a) = r(s,a) + \gamma \widehat{\EE}_{s'
} \min_{i} Q_{\theta^-_i}(s', a'), i\in[1,N],
\label{eq:shared}
\end{equation}
where $\theta_i$ and $\theta^-_i$ are parameters for $i$-th $Q$-network and target $Q$-network, respectively, and $a' \sim \pi(s')$. Additionally, we notice that the results reported in \cite{msg,SAC-N} demonstrate that using shared targets has a prior performance in Mujoco tasks but fails in more challenging domains such as AntMaze. Because shared targets on AntMaze tend to be overly pessimistic, it's harder to explore new out-of-distribution samples, making it more difficult to learn better policies. Thus, for the challenging AntMaze environments, we adhere to previous work and utilize \emph{independent} Bellman targets (refer to Appendix \ref{more discussion} for more details about the theoretical evaluation) without altering the network architecture. Similarly, independent targets can be formulated as:
\begin{equation}
\widehat{\cT} Q_{\theta_i}(s,a) = r(s,a) +  \gamma\widehat{\EE}_{s'
} Q_{\theta^-_i}(s', a'), i\in[1,N].
\label{eq:independent}
\end{equation}

During offline training, the $Q$-ensembles are utilized to learn the value function and update the policy. In online fine-tuning, the learned $Q$-ensembles and policies continue to be updated with online experiences, as shown in Figure~\ref{fig-framework}. Several online RL methods \shortcite{redq,sunrise} also employ $Q$-ensembles and suggest maximizing the average $Q$-values for policy optimization. In our study, we have found that maximizing the minimum $Q$-values, consistent with the objective in the offline phase, is also highly effective in obtaining the optimal policy during online fine-tuning. 

\subsubsection{Robustness Regularization} 
Similar to previous studies \shortcite{s4rl,sr2l} that consider robustness in RL, our goal is to enhance the robustness of offline-trained value function and policy, which can have large estimation bias caused by the distribution shift in the fine-tuning phase. Different from previous offline RL approaches that mainly mitigate the effect of perturbed actions \shortcite{adaptivebc,cal-ql}, we adopt a different robustness perspective by suggesting that the distributional shift in offline-to-online process brings both OOD states and actions.

Intuitively, the offline pre-trained policy inevitably encounters OOD samples during the offline-to-online process. Since the process of online exploration is based on the rollout of the offline policy, these OOD samples tend to be distributed around the offline data. In order to deal with the problem of distribution shift, we attempt to introduce additional adversarial samples for smoothness and uncertainty estimation. In this way, we can guarantee that the policies and $Q$-values of in-sample data and perturbed samples do not deviate too much and smooth out them within a small range beyond the dataset's state space, thereby leading to smooth value function and policy that are robust to distribution shift. To this end, we follow a similar approach as in RORL \shortcite{rorl} to construct adversarial samples for regularization in the offline-to-online process.

\paragraph{Smoothness Regularization} We employ regularization on both policy and $Q$-function by minimizing the difference between estimations obtained from in-sample data and perturbed samples. Specifically, we synthesize perturbed samples by constructing a perturbation set $\mathbb{B}_d(s,\epsilon)=\{\hat{s}:d(s,\hat{s})\leq \epsilon\}$ for state $s$, which is an $\epsilon$-radius ball with a distance metric $d(\cdot,\cdot)$. By sampling from this set $\hat{s}\in \mathbb{B}_d(s,\epsilon)$, the proposed RO2O minimizes the difference between $Q_{\theta_i}(s,a)$ and $Q_{\theta_i}(\hat{s},a)$, as
\begin{equation}
\nonumber
\mathcal{L}^i_{\rm{Qsmooth}}=\mathcal{L} \left(Q_{\theta_i}(s,a), Q_{\theta_i}(\hat{s},a)\right),
\end{equation}
which enforces value smoothness to adversarial state $\hat{s}$, and $\mathcal{L}$ can be a $L_2$ distance. The smooth loss should be applied to each network in the ensemble. To simplify the optimization, we choose to minimize the maximal smooth loss $\max\nolimits_{i} \mathcal{L}^i_{\rm{Qsmooth}}$ among the ensemble. Similarly, we can minimize the difference between $\pi(a|s)$ and $\pi(a|\hat{s})$, which is realized by minimizing the Jensen-Shannon divergence $D_{\rm JS}(\pi(\cdot|s)\|\pi(\cdot|\hat{s}))$. The JS divergence $D_{\rm JS}(\pi(\cdot|s) | \pi(\cdot|\hat{s}))$ is defined as: $D_{\rm JS}(\pi(\cdot|s)||\pi(\cdot|\hat{s})) = \frac{1}{2}D_{\rm KL}(\pi(\cdot|s)||M) + \frac{1}{2}D_{\rm KL}(\pi(\cdot|\hat{s})||M),$ where $M$ is the mixture distribution of $\pi(\cdot|s)$ and $\pi(\cdot|\hat{s})$, given by $M=\frac{1}{2}(\pi(\cdot|\hat{s}) + \pi(\cdot|s))$.


\paragraph{Overestimation Penalty} Meanwhile, since the $Q$-values for OOD states and actions can be overestimated, we penalize their value estimation with uncertainty quantification following prior works \shortcite{pbrl,rorl}. For OOD states $\hat{s}\in \mathbb{B}_d(s,\epsilon)$ and OOD actions $\hat{a}\sim \pi(\hat{s})$, their pseudo Bellman targets can be expressed as $\widehat{\mathcal{T}}^{\rm ood}Q(\hat{s},\hat{a})=Q(\hat{s},\hat{a}) - \alpha U(\hat{s},\hat{a})$. Here, $\theta_i$ denotes the parameters of the $i$-th Q-function. We define a loss function to constrain the value of OOD samples as:
\begin{equation}
\nonumber
\mathcal{L}^i_{\rm{ood}}=\mathbb{E}_{\hat{s}\in \mathbb{B}_d(s,\epsilon), \hat{a}\sim \pi(\hat{s})} (\widehat{\mathcal{T}}^{\rm ood} Q_{\theta_i}(\hat{s},\hat{a}) - Q_{\theta_i}(\hat{s},\hat{a}))^2.
\end{equation}

Specifically, we define the uncertainty function $U(\hat{s}, \hat{a})$ as follows: $$U(\hat{s}, \hat{a}) = \sqrt{\frac{1}{K} \sum^K_{k=1} (Q_{\theta_i}(\hat{s},\hat{a})-\bar{Q}_{\theta_i}(\hat{s},\hat{a}))^2},$$ where $K$ is the number of the ensemble networks and $\bar{Q}_{\theta_i}(\hat{s},\hat{a})$ means the mean $Q$-value of the ensemble networks.

\subsubsection{Algorithm Description} 
As outlined in Algorithm~\ref{alg:algorithm}, the learning process encompasses two phases: offline pre-training and online fine-tuning. We adopt the SAC \shortcite{sac} algorithm as our backbone. For $Q$-value functions, RO2O has the following loss function: 
\begin{equation}
    \mathcal{L}^i_Q = \mathbb{E}_{(s,a,r,s')\sim \cD} \left[ \mathcal{L}^i_{\rm{TD}} + \eta_1 \mathcal{L}^i_{\rm{Qsmooth}} + \eta_2 \mathcal{L}^i_{\rm{ood}}\right],
    \label{eq:lq}
\end{equation}
where $\cD=\cD_{\rm offline}$ in the offline training phase and $\cD=\cD_{\rm offline}\cup \cD_{\rm online}$ in the online fine-tuning phase. $\mathcal{L}^i_{\rm{TD}}=(\mathcal{T}Q_{\theta_i}(s,a)-Q_{\theta_i}(s,a))^2$ represents the TD error, where $\mathcal{T}Q_{\theta_i}=r+\gamma \left(\min Q_{\theta_i^-}(s',a')-\beta \log \pi(a'|s')\right)$ when taking shared targets in Equation~\eqref{eq:shared} for Mujoco tasks, and $\mathcal{T}Q_{\theta_i}=r+\gamma \left(Q_{\theta_i^-}(s',a')-\beta \log \pi(a'|s')\right)$ when using independent targets in Equation~\eqref{eq:independent} for AntMaze tasks. The policy is learned by optimizing the following loss function:
\begin{equation}
\begin{aligned}
    \mathcal{L}_\pi &= \mathbb{E}_{(s,a)\sim \cD}\bigl[\min\limits_{i} Q_{\theta_i}(s,a) + \beta \log \pi_{\zeta}(a|s) \bigr. \bigl. + \eta_3 D_{\rm JS}(\pi_{\zeta}(\cdot|s)\|\pi_{\zeta}(\cdot|\hat{s}))\bigr],
     \label{eq:lp}
\end{aligned}
\end{equation}
where $\phi$ represents the parameters of the policy network. 
In Equation~\eqref{eq:lp}, the first term maximizes the minimum of the ensemble $Q$-functions to obtain a conservative policy, the second term is the entropy regularization, and the third term is the smooth constraint. We remark that we maintain the same loss function throughout the offline-to-online process, which is more elegant than previous offline-to-online methods. The difference between the pre-training and the fine-tuning phase lies in the data sampled to estimate of the expectations in Equation~\eqref{eq:lq} and Equation~\eqref{eq:lp}. In implementation, we also apply normalization to states, which is widely used in previous work \shortcite{td3bc,stablebaselines3}. This also helps to deal with the variations of states in the fine-tuning phase. 

\begin{algorithm}[tb]
\caption{Robust Offline-to-Online RL algorithm}
\label{alg:algorithm}
\textbf{Require:} ensemble $Q$-networks $\{Q_{\theta_i}\}_{i=1}^n$, target networks $\{Q_{\theta^-_i}\}_{i=1}^n$, and policy network $\pi_\phi$
\begin{algorithmic}[1] 
\STATE \textit{// Offline Pre-training}
\WHILE{$t\leq T_1$} 
\STATE Sample mini batches from $\mathcal{D}$.
\STATE Calculate robustness regularization $\mathcal{L}^i_{\rm{Qsmooth}}, \mathcal{L}^i_{\rm{ood}}$.
\STATE Update $Q$-functions with Equation~\eqref{eq:lq} and update target networks softly.
\STATE Update the policy with Equation~\eqref{eq:lp}.
\ENDWHILE
\STATE \textit{// Online Fine-tuning}
\WHILE{$t\leq T_2$}
\STATE Interact with the online environment with $\pi_\phi$.
\STATE Collect transitions into new buffer $\mathcal{B}$.
\STATE Sample batches from buffer $\mathcal{B}$.
\STATE Update $Q$-functions and the policy with $\mathcal{L}^i_{Q}, \mathcal{L}_{\pi}$.
\ENDWHILE
\end{algorithmic}
\end{algorithm}

\subsection{Theoretical Analysis}

Our analyses are conducted in linear MDP assumption \shortcite{lsvi-2020,jin2021pessimism}, where the transition kernel and the reward function are linear in a given state-action feature $\phi(s,a)$. We estimate the value function by $Q(s,a)\approx \hw^{\top}\phi(s,a)$. See the appendix for the details. 

We start by considering the offline training phase where the value function is learned from $\cD_{\rm offline}$. According to the loss in Equation~\eqref{eq:lq}, the parameter $\hw$ can be solved by 
\begin{equation}
\label{eq:simplified_problem}
\begin{aligned}
\widetilde w_{\rm offline} =  \min_{w\in \mathcal{R}^d}  &\Big[\sum_{i=1}^{m} \big(y_t^i-Q_{w}(s_t^i,a_t^i)\big)^2 + \sum_{(\hat s, \hat a, \hat y) \sim \mathcal{D}_{\rm{ood}} } \big(\hat y - Q_{w}(\hat s,\hat a)\big)^2 \\ & + \sum_{i=1}^{m} \frac{1}{|\mathbb{B}_d(s_t^i,\epsilon)|} \sum_{(\hat{s}_t^i,s_t^i,a_t^i)\in \mathcal{D}_{\rm{robust}}}  \big(Q_{w}(s_t^i,a_t^i) - Q_{w}(\hat{s}_t^i,a_t^i)\big)^2\Big],
\end{aligned}
\end{equation}
where we denote $y=\widehat{\cT}Q$ and $\hat{y}=\widehat{\cT}^{\rm ood}Q$ as the learning targets for simplicity. The three terms in Equation~\eqref{eq:simplified_problem} correspond to TD-loss, OOD penalty, and smoothness constraints, respectively. For the clarity of notations, we explicitly define a dataset $\cD_{\rm ood}$ for OOD sampling, and an adversarial dataset $\cD_{\rm robust}$ for the smoothness term. Following Least-Squares Value Iteration (LSVI) \shortcite{lsvi-2020}, the solution of Equation~\eqref{eq:simplified_problem} takes the following form as
\begin{equation}
\label{eq::w_ood_solu}
    \widetilde w_t =\widetilde \Lambda_t^{-1} \Big( \sum_{i=1}^{m} \phi(s_t^i,a_t^i) y_t^i + \sum_{(\hat s, \hat a, \hat y) \sim \mathcal{D}_{\rm{ood}} }\phi(\hat s,\hat a)  \hat y  \Big),
\end{equation}
and the covariance matrix $\widetilde{\Lambda}_t$ is
\begin{equation}
\label{eq:covariance-rorl}
\widetilde \Lambda_t = \widetilde \Lambda^{\rm in}_t + \widetilde \Lambda^{\rm ood}_t+ \widetilde \Lambda^{\rm robust}_t,
\end{equation}
where the first term $\widetilde \Lambda^{\rm in}_t=\sum_{i=1}^{m} \phi(s_t^i,a_t^i)\phi(s_t^i,a_t^i)^\top$ is calculated on in-distribution (i.e., in $\cD_{\rm offline}$) data, the second term is $\widetilde \Lambda^{\rm ood}_t=\sum_{\mathcal{D}_{\rm{ood}}} \phi(\hat s_t,\hat a_t)\phi(\hat s_t,\hat a_t)^\top$ is calculated on OOD samples (i.e., in $\cD_{\rm ood}$), and the third term is calculated on adversarial samples (i.e., in $\cD_{\rm robust}$), as
$\widetilde \Lambda^{\rm robust}_t = \sum_{i=1}^{m} \frac{1}{|\mathbb{B}_d|}  \sum_{\mathcal{D}_{\rm{robust}}} \big[\phi(\hat s,a) - \phi(s,a)\big]\big[\phi(\hat s,a) - \phi(s,a)\big]^\top$.

For comparison, we consider a variant of RO2O without smoothness regularization. Following LSVI, the solution of this variant takes a similar form as Equation~\eqref{eq::w_ood_solu}, but with a different covariance matrix as $\widetilde \Lambda^{\rm in}_t+\widetilde \Lambda^{\rm ood}_t$. The difference in covariance matrices originates from the additional adversarial samples in RO2O. We denote the dataset for RO2O as $\cD_{\rm RO2O}=\cD_{\rm offline}\cup \cD_{\rm ood}\cup \cD_{\rm robust}$, and for this variant as $\cD_{\rm variant}=\cD_{\rm offline}\cup \cD_{\rm ood}$ without smoothness constraints. 

Following the theoretical framework in PEVI \shortcite{jin2021pessimism}, the sub-optimality gap of offline RL algorithms with uncertainty penalty is upper-bounded by the lower-confidence-bound (LCB) term, defined by
\begin{equation}
\nonumber
\Gamma^{\rm lcb}(s_t,a_t;\cD)= \beta_t \big[\phi(s_t,a_t)^\top\Lambda_t^{-1}\phi(s_t,a_t)\big]^{\nicefrac{1}{2}},
\end{equation}
where the form of $\Lambda_t$ depends on the learned dataset (e.g., $\cD_{\rm RO2O}$ or $\cD_{\rm variant}$), and $\beta_t$ is a factor. Then the following theorem shows our smoothness regularization leads to smaller uncertainties for arbitrary state-action pairs, especially for OOD samples (e.g., from online interactions).

\begin{theorem}\label{thm:1}
Assuming that the size of adversarial samples $\mathbb{B}_d(s^i_t,\epsilon)$ is sufficient and the Jacobian matrix of $\phi(s, a)$
has full rank, the smoothness constraint leads to smaller uncertainty for $\forall (s^{\star},a^{\star})\in \cS\times\cA$, as 
\begin{equation}
\nonumber
\Gamma^{\rm lcb}(s^{\star},a^{\star};\cD_{\rm RO2O}) < \Gamma^{\rm lcb}(s^{\star},a^{\star};\cD_{\rm variant}),
\end{equation}
where the covariance matrices for these two LCB terms are $\widetilde \Lambda_t$ in Equation~\eqref{eq:covariance-rorl} and $\widetilde \Lambda^{\rm in}_t+\widetilde \Lambda^{\rm ood}_t$, respectively.
\label{thm:distribution_shift}
\end{theorem}

According to Theorem~\ref{thm:1}, with the additional term $\widetilde \Lambda^{\rm robust}_t = \sum_{i=1}^{m} \frac{1}{|\mathbb{B}_d|}  \sum_{\mathcal{D}_{\rm{robust}}} \big[\phi(\hat s,a) - \phi(s,a)\big]\big[\phi(\hat s,a) - \phi(s,a)\big]^\top$ in the covariance matrix $\widetilde \Lambda_t$ of $\cD_{\rm RO2O}$, the uncertainty of OOD samples 
 measured by UCB will be reduced. As an extreme example in tabular case, the uncertainty for a purely OOD $(s^{\star},a^{\star})$ pair can be large as $\Gamma^{\rm lcb}(s^{\star},a^{\star};\cD_{\rm variant})\to \infty$ without the smoothness term, while $\Gamma^{\rm lcb}(s^{\star},a^{\star}; \cD_{\rm RO2O})\leq \nicefrac{\beta_t}{\sqrt{\lambda}}$ with $\lambda>0$. As a result, RO2O is more robust to significant distribution shift theoretically. See appendix for the proof.


Then, for online fine-tuning with new data from $\cD_{\rm online}$, the following theorem shows RO2O can consistently reduce the sub-optimality gap with online fine-tuning, as 
\begin{theorem}
Under the same conditions as Theorem~\ref{thm:distribution_shift}, with additional online experience in the fine-tuning phase, the sub-optimality gap holds for RO2O in linear MDPs, as
\begin{equation}
\nonumber
\begin{aligned}
{\rm SubOpt} (\pi^*, \widetilde\pi) &\leq \sum\nolimits_{t=1}^{T} \mathbb{E}_{\pi^*} \big[ \Gamma_i^{\rm lcb}(s_t,a_t;\cD_{\rm RO2O}\cup 
\cD_{
\rm online
}) \big] \\
&\leq \sum\nolimits_{t=1}^{T} \mathbb{E}_{\pi^*} \big[ \Gamma_i^{\rm lcb}(s_t,a_t;\cD_{\rm RO2O}) \big],
\end{aligned}
\end{equation}
where $\widetilde\pi$ and $\pi^{*}$ are the learned policy and the optimal policy in $\cD_{\rm RO2O}\cup 
\cD_{
\rm online
}$, respectively.
\end{theorem}
Theorem~\ref{thm:optimality} shows that the optimality gap shrinks if the data
coverage of $\pi^*$ is better. See appendix for the proof. Considering a sub-optimal dataset is used in offline training, via interacting and learning in online fine-tuning, the agent is potential to obtain high-quality data to consistently reduce the sub-optimality and result in a near-optimal policy. 


\section{Experiments}
We present a comprehensive evaluation of RO2O in the context of the Offline-to-Online RL setting. Specifically, we investigate whether RO2O can perform favorable offline training and further policy improvement given limited interactions. We compare RO2O to existing offline RL algorithms in offline pre-training and also compare it to offline-to-online algorithms in online adaptation. We also conduct ablation studies and visualizations to illustrate the effectiveness of our method.

\subsection{Setups and Baselines}
Our experiments are conducted on challenging environments from the D4RL \shortcite{d4rl} benchmark, specifically focusing on the Mujoco and AntMaze tasks. These environments are carefully selected to simulate real-world scenarios with limited data availability. We compare RO2O with the following RL algorithms: (\romannumeral1) PEX \shortcite{pex}: A recent offline-to-online method that performs policy expansion has shown promising results in longer online interaction steps. (\romannumeral2) AWAC \shortcite{awac}: An efficient algorithm that employs an advantage-weighted actor-critic framework, which is one of the earlier methods to achieve policy improvement during the online fine-tuning phase. (\romannumeral3) IQL \shortcite{iql}: A state-of-art offline algorithm that attempts to conduct in-sample learning and expected regression. (\romannumeral4) Cal-QL \shortcite{cal-ql}: An efficient algorithm calibrates $Q$-values within a reasonable range to improve policy performance. (\romannumeral5) SPOT \shortcite{SPOT}: An algorithm utilizes density regularization to limit the difference between the learning policy and the current policy. (\romannumeral6) SAC \shortcite{sac}: A SAC agent trained from scratch which highlights the benefit of Offline-to-Online RL, as opposed to fully online RL, in terms of learning efficiency.

\subsection{Performance Comparisons}

We conducted our comparisons using multiple offline datasets and tasks. In this study, we exclude the random datasets, as in typical real-world scenarios, we rarely use a random policy for system control. For the Mujoco locomotion tasks, we conducted 2.5M training steps over all datasets during the offline pre-training phase. Then, we proceeded with online fine-tuning, involving an additional 250K environment interactions. For the AntMaze navigation tasks, we performed 1M training steps on six types of datasets with different complexities and qualities, followed by 250K additional online interactions. It is worth noting that we all use an ensemble size of 10 for training in both tasks. More details about experiments and implementations are introduced in the Appendix \ref{app:environment} and \ref{app:implement}.

\paragraph{Offline performance on MuJoCo locomotion tasks}
First of all, we evaluate the performance of each method on Mujoco locomotion tasks, which include three environments: \textit{Halfcheetah}, \textit{Walker2d}, and \textit{Hopper}. Different types of datasets are selected for offline pre-training, including medium, medium-replay, medium-expert, and expert datasets. Table~\ref{table1} reports the offline performance of the average normalized score across five seeds. Compared to other algorithms, RO2O has certain superiority in the offline training phase.

\begin{table}[ht]
\centering
\tiny
\caption{Offline performance on MuJoCo locomotion tasks.}
\label{table1}
\resizebox{\textwidth}{!}{
\begin{tabular}{lllllll}
\hline
\textbf{Environment} & PEX & AWAC & IQL & SPOT & Cal-QL & RO2O \\ \hline
halfcheetah-medium                    & 48.67 ± 0.15  & 50.00 ± 0.27  & 48.33 ± 0.35  & 46.78 ± 0.50   & 47.75 ± 0.38  & \textbf{66.08 ± 0.45}  \\
halfcheetah-medium-replay             & 44.57 ± 0.47  & 45.28 ± 0.31  & 43.75 ± 0.97  & 43.29 ± 0.47   & 46.26 ± 0.57  & \textbf{60.89 ± 1.01}  \\
halfcheetah-medium-expert             & 78.9 ± 11.77  & 94.73 ± 0.64  & 94.19 ± 0.30  & 94.70 ± 1.02   & 67.14 ± 7.40  & \textbf{104.73 ± 2.07} \\
halfcheetah-expert                    & 91.2 ± 4.43   & 97.57 ± 0.94  & 97.11 ± 0.19  & 95.21 ± 0.93   & 96.5 ± 0.66   & \textbf{104.08 ± 1.66} \\ \hline
walker2d-medium                       & 61.87 ± 2.06  & 84.24 ± 1.15  & 83.96 ± 2.68  & 56.81 ± 3.96   & 64.07 ± 7.61  & \textbf{103.25 ± 1.67} \\
walker2d-medium-replay                & 38.4 ± 13.36  & 80.92 ± 1.70  & 77.28 ± 7.45  & 70.49 ± 22.57  & \textbf{94.48 ± 6.44}  & 93.05 ± 4.74  \\
walker2d-medium-expert                & 98.8 ± 4.78   & 112.62 ± 0.66 & 111.24 ± 0.92 & 77.58 ± 8.49   & 108.26 ± 5.56 & \textbf{120.01 ± 0.70} \\
walker2d-expert                       & 103.13 ± 6.69 & 91.66 ± 35.78 & 112.67 ± 0.21 & 105.13 ± 13.03 & 111.93 ± 0.24 & \textbf{112.84 ± 3.42} \\ \hline
hopper-medium                         & 51.3 ± 5.07   & 71.33 ± 8.80  & 56.33 ± 2.83  & 82.25 ± 2.16   & 83.34 ± 0.91  & \textbf{104.95 ± 0.03} \\
hopper-medium-replay                  & 77.9 ± 5.77   & 96.56 ± 2.23  & 82.55 ± 17.57 & 70.37 ± 12.51  & 85.59 ± 1.85  & \textbf{103.77 ± 0.47} \\
hopper-medium-expert                  & 46.73 ± 48.88 & 108.36 ± 3.12 & 85.21 ± 39.43 & 96.52 ± 12.03  & 108.82 ± 0.21 & \textbf{112.69 ± 0.02} \\
hopper-expert                         & 102.27 ± 6.11 & 103.88 ± 8.98 & 100.36 ± 9.96 & 110.00 ± 0.41  & 107.29 ± 3.50 & \textbf{112.31 ± 0.01} \\
\hline
\end{tabular}}
\end{table}

\paragraph{Fine-tuning performance on Mujoco locomotion tasks}
Figure~\ref{fig:mujoco-ft} illustrates the fine-tuning performance of different methods on Mujoco locomotion tasks. Compared with pure online RL methods such as SAC, other methods almostly reflect the advantages of offline pre-training. Within limited online interactions, IQL, AWAC and SPOT fail to achieve effective policy improvement, while PEX suffers from the performance drop. For PEX, we speculate that it is due to the randomness of strategies expanded in the online phase, which could lead to a poor initial strategy and requires lots of interactions to improve its performance. For Cal-QL, it can achieve effective performance improvement in most tasks, but the improvement is still relatively limited. In contrast, RO2O exhibits a significant improvement in performance during the fine-tuning process and requires fewer steps to achieve the highest score. Compared with them, RO2O showcases comparable or better performance with 250K fine-tuning steps, indicating the efficiency and superiority.

\begin{figure*}[!t]
\centering
\includegraphics[width=\textwidth]{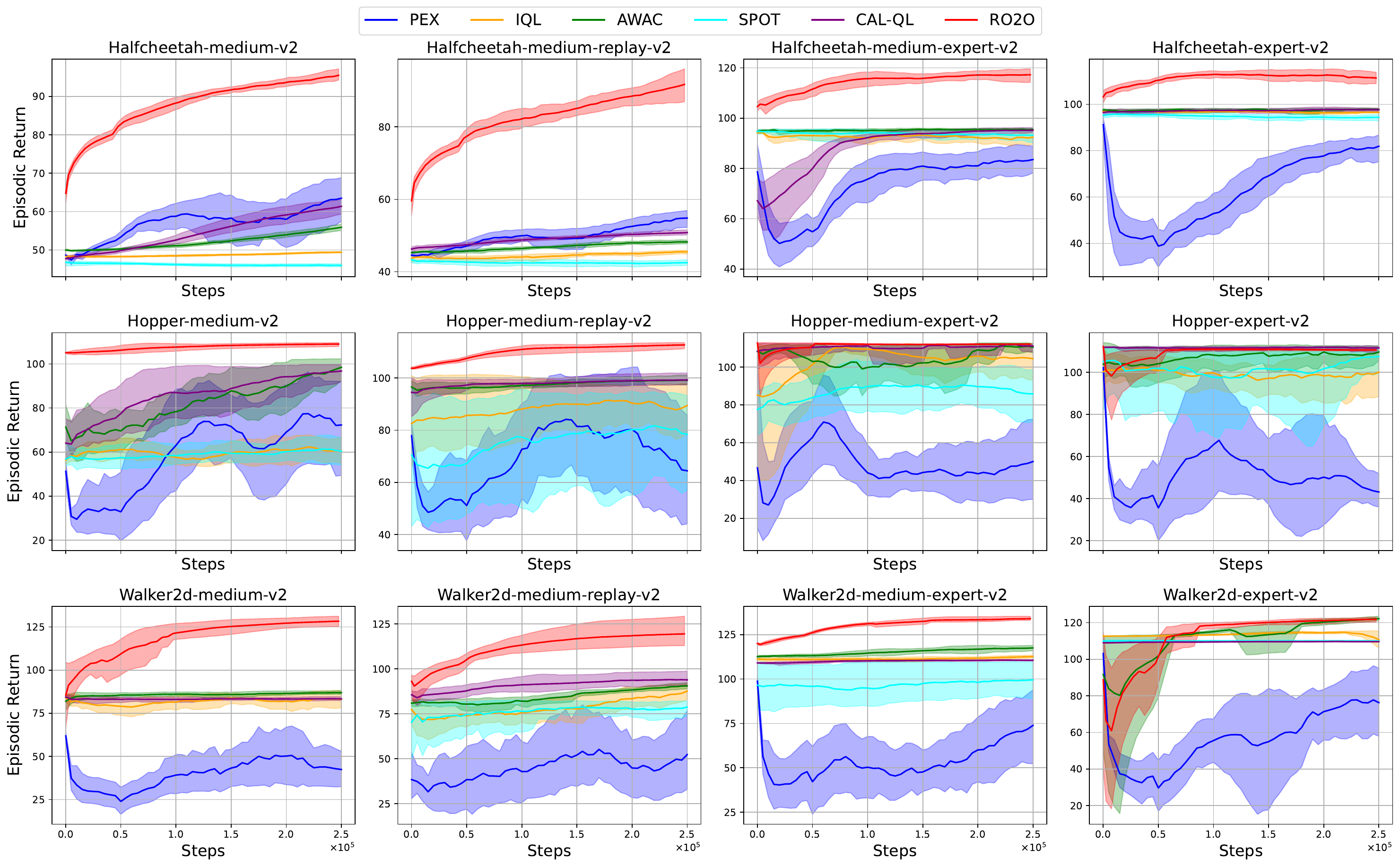}
\caption{Fine-tuning performance curves of different methods across five seeds on MuJoCo locomotion tasks. The mean and standard deviation are shown by the solid lines and the shaded areas, respectively.}
\label{fig:mujoco-ft}
\end{figure*}

\paragraph{Offline performance on AntMaze navigation tasks}
We also perform evaluations on the challenging AntMaze navigation tasks, where agents must learn to control the robot and stick trajectories together given sparse rewards. Agents are pre-trained on six types of datasets with different complexity and quality. Considering the poor performance of AWAC and SAC in this task, we have opted not to compare our approach with them. Table~\ref{tab:antmaze-offline} reports the normalized scores using different methods across five seeds. We observe that RO2O achieves the best performance on almost all tasks.

\begin{table}[ht]
\centering
\caption{Offline performance on the challenging AntMaze environment.}
\label{tab:antmaze-offline}
\resizebox{\columnwidth}{!}{
\begin{tabular}{lllllll}
\hline
\textbf{Environments} & PEX & IQL & SPOT & Cal-QL & RO2O \\
\hline
antmaze-umaze & 87.33 ± 4.04 & 77.0 ± 6.38 & 89 ± 5.29 & 65.75 ± 4.03 & \textbf{93.67 ± 5.77} \\
antmaze-umaze-diverse & 58.67 ± 9.07 & \textbf{65.24 ± 6.40} & 42.75 ± 5.32 & 48.75 ± 4.43 & 63.67 ± 8.02 \\
antmaze-medium-diverse & 72.33 ± 5.13 & 73.75 ± 6.30 & 74.25 ± 4.99 & 1.25 ± 0.96 & \textbf{91.67 ± 5.13} \\
antmaze-medium-play & 68 ± 6.56 & 66 ± 7.55 & 71.5 ± 8.43 & 0.0 ± 0.0 & \textbf{86.67 ± 3.06} \\
antmaze-large-diverse & 45.67 ± 4.16 & 30.25 ± 4.20 & 36.5 ± 17.62 & 0.0 ± 0.0 & \textbf{65.33 ± 5.71} \\
antmaze-large-play & 51 ± 17.69 & 42.0 ± 5.23 & 30.25 ± 17.91 & 0.25 ± 0.5 & \textbf{61.33 ± 9.82} \\
\hline
\end{tabular}
}
\end{table}

\paragraph{Fine-tuning performance on AntMaze navigation tasks}
Figure~\ref{fig:antmaze-plot} demonstrates the fine-tuning performance on AntMaze tasks. In the fine-tuning phase, IQL and SPOT achieve stable learning but limited improvement, while PEX suffers from a performance drop. Cal-QL rapidly enhances its performance from a poor initial policy, but cannot perform well in large scenarios. Different from these baselines, RO2O achieves significant improvement over all tasks within limited interactions. It provides good initial performance for online fine-tuning within limited interaction steps, demonstrating considerable advantages. In summary, RO2O achieves robust and state-of-the-art performance in comparison to all baseline methods in almost all tasks.

\begin{figure}[ht]
\centering
\includegraphics[width=0.8\textwidth]{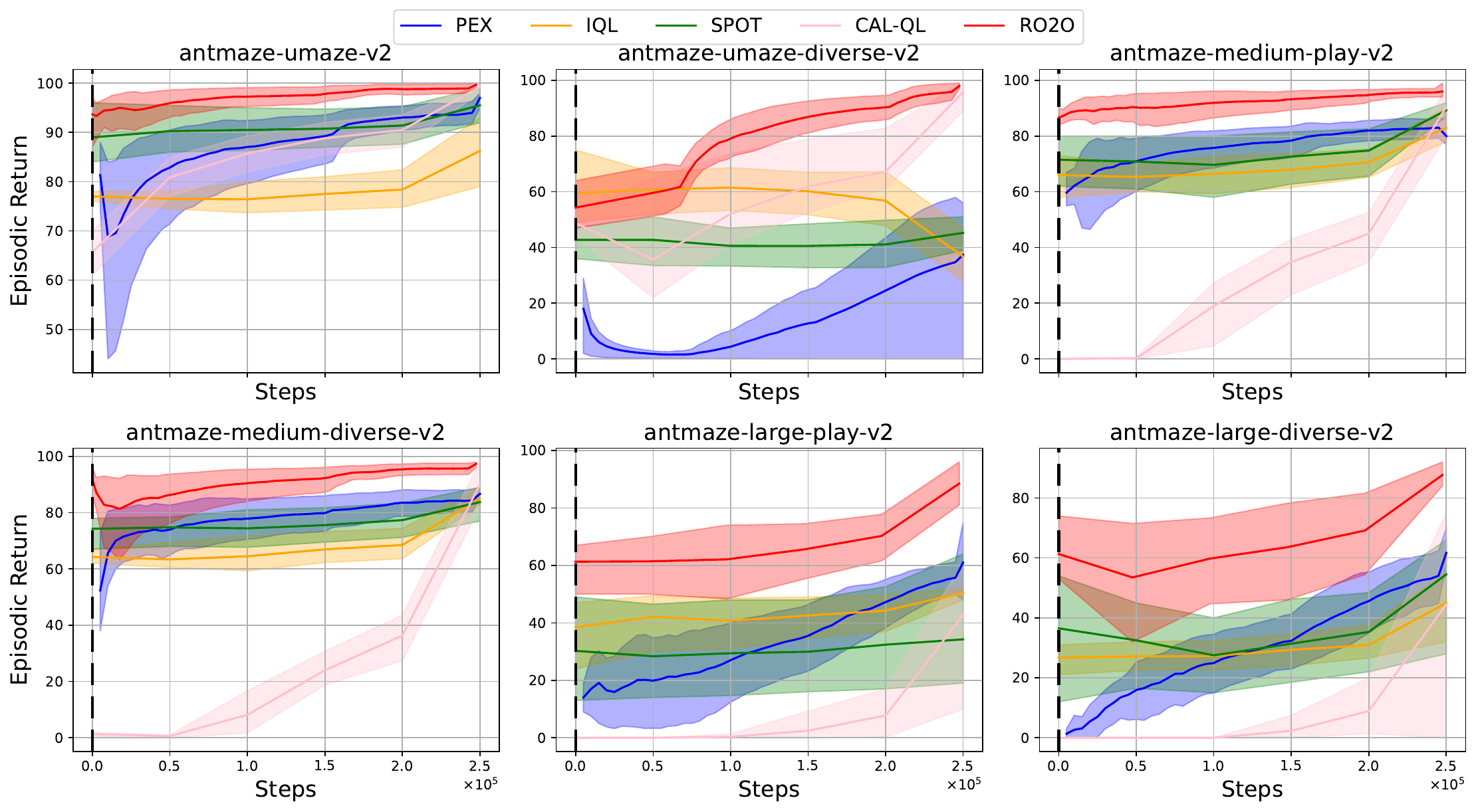} 
\caption{Fine-tuning performance curves of different methods across five seeds on Antmaze navigation tasks. The mean and standard deviation are shown by the solid lines and the shaded areas, respectively.}
\label{fig:antmaze-plot}
\end{figure}

\paragraph{Robustness Analysis} 
RO2O demonstrates robust performance in all environments except for some scenarios, such as AntMaze-large-play and AntMaze-large-diverse. The performance curves seem to exhibit more fluctuating behaviors, albeit with increasing steps, the performances are superior. We believe there are several reasons: (i) Antmaze tasks provide binary rewards, so that similar policies could obtain largely different returns, thus leading to fluctuated performance. While the regularization terms in RO2O aim to present smooth $Q$-functions and policies, their effects can be limited when the environment is complex, as in a large maze. (ii) In Antmaze tasks, RO2O utilizes independent TD targets instead of shared TD targets in the Bellman update of $Q$-values, which in some degree increases the diversity of $Q$-value estimation for OOD actions. During online fine-tuning phase, RO2O leverages the maximum of ensemble estimation of $Q$-values as the TD target to encourage the exploration of policies. This causes that agents tend to choose OOD actions, which may not perform well and also lead to fluctuated performance.

\subsection{Ablation Study and Visualization}
We analyze the effects of the smoothness regularization and OOD sampling terms on the learning process. Specifically, we consider variants of RO2O without policy smoothing, $Q$-smoothing, or OOD penalty. We conduct the ablation studies on walker2d-medium and hopper-medium tasks. Figure~\ref{fig:ablation} demonstrates the experimental results over three random seeds. In the offline process, we observe that OOD penalty is indispensable to prevent divergence caused by OOD actions. However, it becomes trivial in the online phase since new states or actions may lead to high values and better policies. We also find that policy smoothing and $Q$-smoothing are useful in the offline-to-online process to obtain an effective improvement and mitigate the variance of performance. 

\begin{figure}[ht]
\centering
\includegraphics[width=0.45\textwidth]{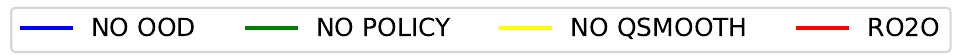}
\includegraphics[width=0.66\textwidth]{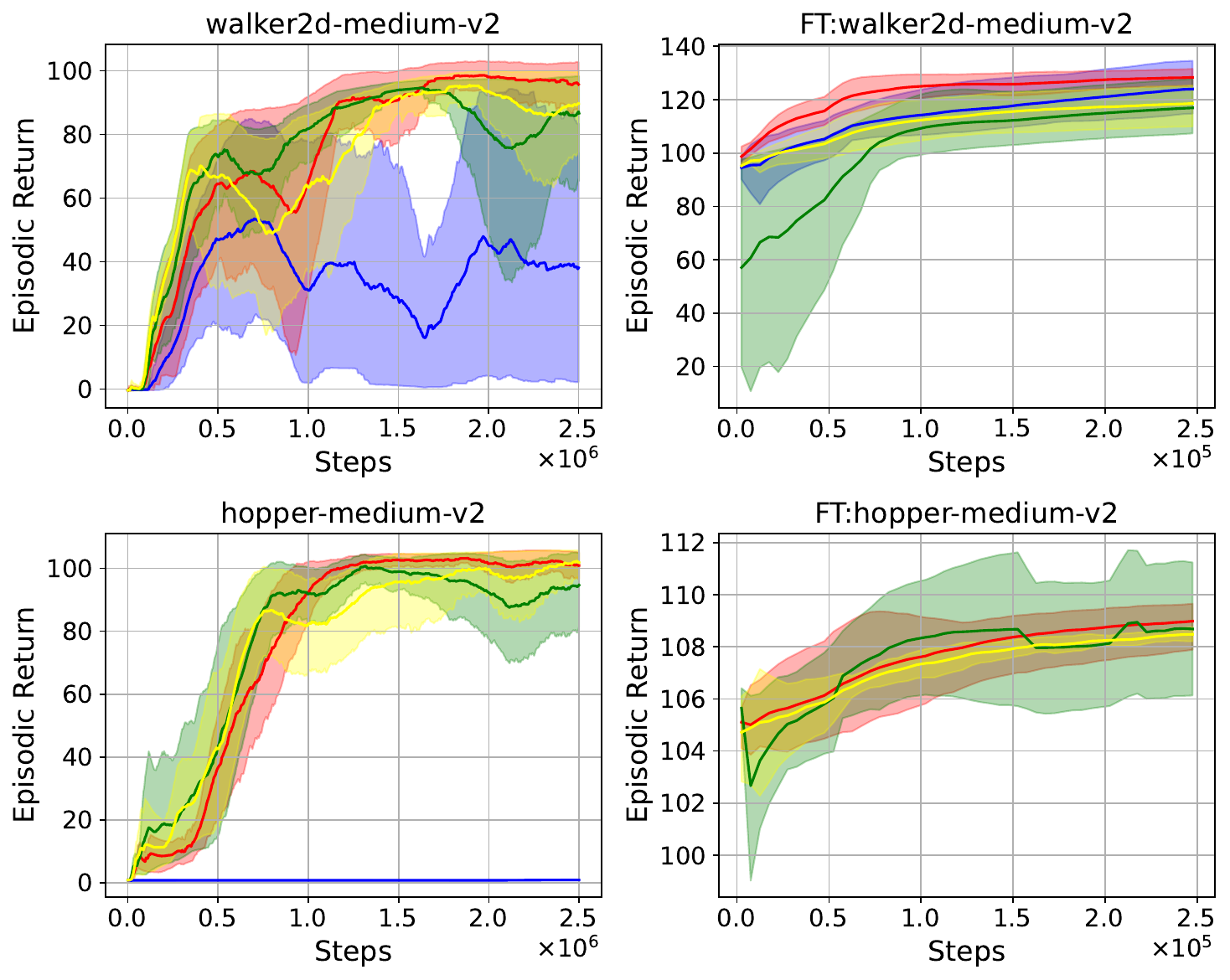}
\caption{Offline (left column) and online performance (right column) when eliminating OOD penalty, policy smoothing, or $Q$-smoothing.}
\label{fig:ablation}
\end{figure}

Additional, we employ various coefficients of the different losses $\eta_1$, $\eta_2$, $\eta_3$ and the radius of the perturbation set $\epsilon$ to investigate the algorithm’s sensitivity to the hyperparameter, where $\eta_1$ maintains a constant value of 0.0001, $\eta_2$ is tuned within \{0.0, 0.1, 0.5\}, $\eta_3$ is searched in \{0.1, 1.0\} and $\epsilon$ is tuned within \{0.0, 0.005, 0.01\}. The results shown in Figure \ref{fig:ablation2} demonstrate that the coefficients of the different losses are a critical factor for both offline and online. We observed that the algorithm is relatively sensitive to the choice of parameters, especially $\eta_3$, where even small changes can lead to significant performance degradation. Constraints on the perturbation set behavior policy may have a greater impact compared to other coefficients. Moreover, in addition to adjusting the coefficients of the different loss functions, we can also modify the radius of the perturbation set to avoid excessive pessimism. Additionally, we notice that 

\begin{figure}[ht]
\centering
\includegraphics[width=0.68\textwidth]{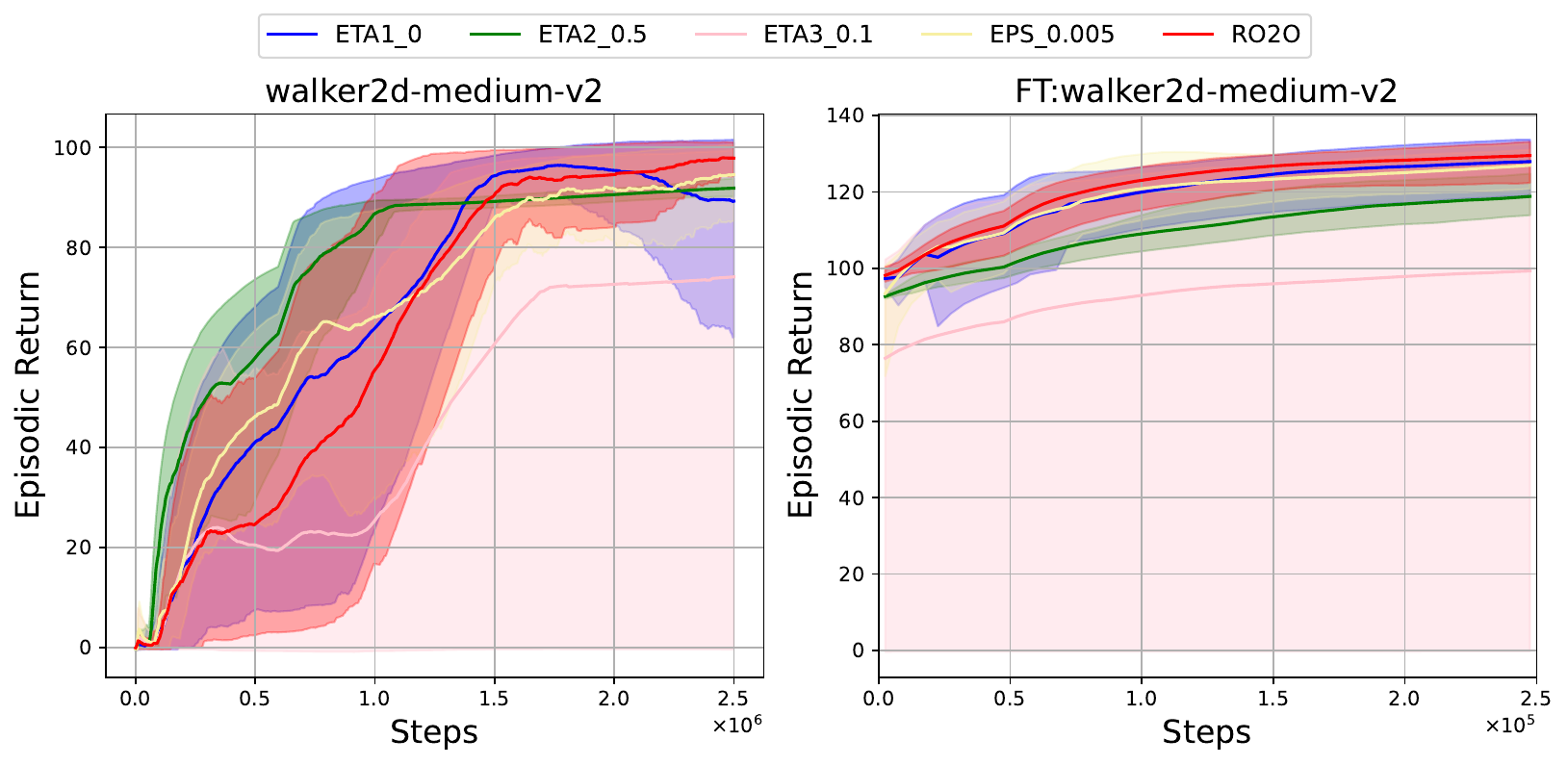}
\caption{Sensitivity analysis of the various coefficients $\eta_1$, $\eta_2$, $\eta_3$ and the radius of the perturbation set $\epsilon$.}
\label{fig:ablation2}
\end{figure}

To better understand the effectiveness of RO2O in the offline-to-online process, we compare the distribution of the offline states and the visited states in online interactions, as shown in Figure~\ref{fig:visualization} (left). The states are visualized via T-SNE. Meanwhile, we use brightness to represent the corresponding reward for each state, as shown in Figure~\ref{fig:visualization} (right). We find that, with consistent policy improvement in online fine-tuning, the agent can obtain high-quality (i.e., with high reward) online experiences that are different from the offline data. Such a phenomenon verifies our theoretical analysis in  Theorem~\ref{thm:optimality}, where RO2O can consistently reduce the sub-optimality gap and improve the policy via online fine-tuning.

\begin{figure}[ht]
    \centering
    \includegraphics[width=0.66\textwidth]{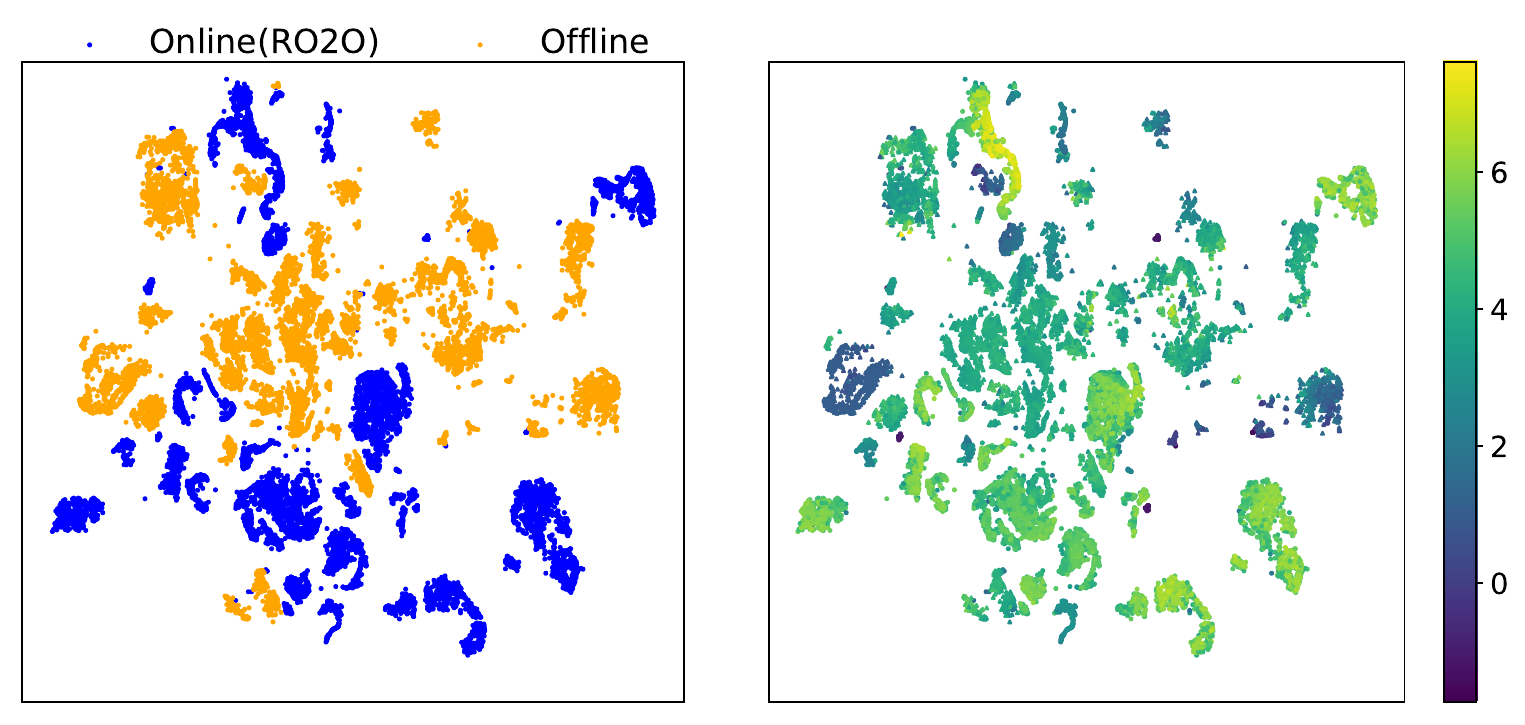}
    \caption{Visualization of the distribution of states (left) and rewards (right) from offline data and online experiences.}
    \label{fig:visualization}
\end{figure}

\section{Computational Cost Comparison}

We compare the computational cost of RO2O against baselines. All methods are run on a single machine with one GPU (NVIDIA GeForce RTX 3090). For each method, we measure the average epoch time (i.e., $1\times10^3$ training steps) and the GPU memory usage on the walker2d-medium-v2 task. The results in Table \ref{tab:runtime} show that although RO2O includes the OOD state-action sampling and the robust training procedure, it does not significantly lag behind other baselines in terms of runtime. And we implemented these procedures efficiently based on the parallelization of $Q$ networks.

\begin{table}[ht]
\footnotesize
\centering
\caption{The computational cost of various algorithms on walker2d-medium-v2.}
\label{tab:runtime}
\begin{tabular}[\textwidth]{ccc}
\hline
       & Runtime (s/epoch) & GPU Memory (GB) \\ \hline
PEX    & 5.18             & 2.17           \\
AWAC   & 8.44             & 5.20           \\
IQL    & 6.31             & 2.38           \\
SAC-10 & 7.82             & 2.23           \\
SPOT   & 5.24             & 5.20           \\
Cal-QL & 15.7             & 2.69           \\
RO2O   & 21.8             & 2.85           \\ \hline
\end{tabular}
\end{table}


\section{Conclusion}

In this paper, we propose RO2O for Offline-to-Online RL by incorporating $Q$-ensembles and smoothness regularization. By regularizing the smoothness of value and policy, RO2O achieves stable offline learning and effective policy improvement in online fine-tuning. Moreover, RO2O maintains the same architecture in the offline-to-online process without specific modifications. Empirical results on Mujoco and AntMaze tasks demonstrate the effectiveness and superiority of RO2O. Future work may explore ways to perform offline-to-online learning with domain gaps, including dynamics or reward differences.

\acks{This work was completed jointly by Xiaoyu Wen and Xudong Yu as co-first authors and was supported by the National Science Fund for Distinguished Young Scholars (Grant No.62025602), the National Natural Science Foundation of China (Grant Nos. 62306242, U22B2036, 11931915), Fok Ying-Tong Education Foundation China (No.171105), the Tencent Foundation, and XPLORER PRIZE.
}

\appendix
\section{Theoretical Analysis}

\subsection{RO2O Algorithm in Linear MDPs}

We consider the loss function of RO2O algorithm in offline learning, which contains temporal-difference (TD) error, smoothness loss, and OOD penalty. Converting the loss function in linear MDPs, the parameter $\hw$ in RO2O can be solved by 
\begin{equation}
\begin{aligned}
\label{eq:app_simplified_problem}
\widetilde w_{\rm offline} =  \min_{w\in \mathcal{R}^d}  &\Big[\sum_{i=1}^{m} \big(y_t^i-Q_{w}(s_t^i,a_t^i)\big)^2 + 
 \sum_{(\hat s, \hat a, \hat y) \sim \mathcal{D}_{\text{ood}} } \big(\hat y - Q_{w}(\hat s,\hat a)\big)^2 
\\& +  \sum_{i=1}^{m} \frac{1}{|\mathbb{B}_d(s_t^i,\epsilon)|} \sum_{(\hat{s}_t^i,s_t^i,a_t^i)\in \mathcal{D}_{\text{robust}}}  \big(Q_{w}(s_t^i,a_t^i) - Q_{w}(\hat{s}_t^i,a_t^i)\big)^2\Big],
\end{aligned}
\end{equation}
where $y=\widehat{\cT}Q$ and $\hat{y}=\widehat{\cT}^{\rm ood}Q$ denote the learning targets for simplicity and $|\mathbb{B}_d(s_t^i,\epsilon)|$ means the size of adversarial samples. The three terms in Equation~\eqref{eq:app_simplified_problem} correspond to TD-loss, OOD penalty, and smoothness constraints, respectively. For the clarity of notations, we explicitly define a dataset $\cD_{\rm ood}$ for OOD sampling, and an adversarial dataset $\cD_{\rm robust}$ for the smoothness constraint. Following LSVI \shortcite{lsvi-2020}, the solution of Equation~\eqref{eq:app_simplified_problem} takes the following form as
\begin{equation}
\label{eq::app_w_ood_solu}
\widetilde w_t =\widetilde \Lambda_t^{-1} \Big( \sum_{i=1}^{m} \phi(s_t^i,a_t^i) y_t^i + \sum_{(\hat s, \hat a, \hat y) \sim \mathcal{D}_{\text{ood}} }\phi(\hat s,\hat a)  \hat y  \Big), 
\end{equation}
and the covariance matrix $\widetilde{\Lambda}_t$ is
\begin{equation}
\begin{aligned}
\label{eq:covariance-rorl2}
\widetilde \Lambda_t &= \widetilde \Lambda^{\rm in}_t + \widetilde \Lambda^{\rm ood}_t+ \widetilde \Lambda^{\rm robust}_t \\& =\sum_{i=1}^{m} \phi(s_t^i,a_t^i)\phi(s_t^i,a_t^i)^\top+\sum_{\mathcal{D}_{\text{ood}}} \phi(\hat s_t,\hat a_t)\phi(\hat s_t,\hat a_t)^\top \\& +
\sum_{i=1}^{m} \frac{1}{|\mathbb{B}_d(s_t^i,\epsilon)|}  \sum_{\mathcal{D}_{\text{robust}}} \big[\phi(\hat s_t^i,a_t^i) - \phi(s_t^i,a_t^i)\big]\big[\phi(\hat s_t^i,a_t^i) - \phi(s_t^i,a_t^i)\big]^\top,
\end{aligned}
\end{equation}
where the first term $\widetilde \Lambda^{\rm in}_t$ is calculated on in-distribution (i.e., in $\cD_{\rm offline}$) data, the second term is $\widetilde \Lambda^{\rm ood}_t$ is calculated on OOD samples (i.e., in $\cD_{\rm ood}$), and the third term is calculated on adversarial samples (i.e., in $\cD_{\rm robust}$), .

For comparison, we consider a variant of RO2O without smoothness regularization, and denote it as `variant'. The parameter of this variant can be solved by 
\begin{equation}
\label{eq:app_variant}
\widetilde w_{\rm variant} =  \min_{w\in \mathcal{R}^d}  \Big[\sum_{i=1}^{m} \big(y_t^i-Q_{w}(s_t^i,a_t^i)\big)^2 + 
\sum_{(\hat s, \hat a, \hat y) \sim \mathcal{D}_{\text{ood}} } \big(\hat y - Q_{w}(\hat s,\hat a)\big)^2 \Big],
\end{equation}

Following LSVI, the solution of this variant takes a similar form as Equation~\eqref{eq::app_w_ood_solu}, but with a different covariance matrix as 
\begin{equation}
\nonumber
\widetilde \Lambda^{\rm variant}_t = \widetilde \Lambda^{\rm in}_t+\widetilde \Lambda^{\rm ood}_t.
\end{equation}

We remark that the difference in covariance matrices between RO2O and this variant originates from the additional adversarial samples from $\cD_{\rm robust}$ used in RO2O. We denote the dataset for RO2O as 
\[\cD_{\rm RO2O}=\cD_{\rm offline}\cup \cD_{\rm ood}\cup \cD_{\rm robust},
\] 
and for this variant as 
\[
\cD_{\rm variant}=\cD_{\rm offline}\cup \cD_{\rm ood}
\]
without smoothness constraints. 

\subsection{Effective of Smoothness with LCB}

Following the theoretical framework in PEVI \shortcite{jin2021pessimism}, the sub-optimality gap of offline RL algorithms with uncertainty penalty is upper-bounded by the lower-confidence-bound (LCB) term, defined by
\begin{equation}
\nonumber
\Gamma^{\rm lcb}(s_t,a_t;\cD)= \beta_t \big[\phi(s_t,a_t)^\top\Lambda_t^{-1}\phi(s_t,a_t)\big]^{\nicefrac{1}{2}},
\end{equation}
which forms an uncertainty quantification with the covariance matrix $\Lambda_i^{-1}$ given the dataset $\cD_i$ \shortcite{lsvi-2020,jin2021pessimism}, and the form of $\Lambda_t$ depends on the learned dataset (e.g., $\cD_{\rm RO2O}$ or $\cD_{\rm variant}$). $\beta_t$ is a factor. LCB measures the confidence interval of $Q$-function learned by the given dataset. Intuitively, $\Gamma^{\rm lcb}_i(s,a)$ can be considered as a reciprocal pseudo-count of the state-action pair in the representation space. 

In the following, we aim to show the smoothness regularization leads to smaller uncertainties for arbitrary state-action pairs, especially for OOD samples (e.g., from online interactions). We start by building a Lemma to show the covariance matrix $\widetilde{\Lambda}^{\rm robust}_t$ introduced by smoothness regularization calculated in $\cD_{\rm robust}$ is positive-definite.

\begin{lemma}
Assuming that the size of adversarial samples $\mathbb{B}_d(s^i_t,\epsilon)$ is sufficient and the Jacobian matrix of $\phi(s, a)$
has full rank, then the covariance matrix $\widetilde \Lambda_t^{\rm robust}$ is positive-definite: $\widetilde \Lambda_t^{\rm robust} \succeq \lambda \cdot \mathrm{I}$ where $\lambda > 0$.
\label{tm:PD_matrix}
\end{lemma}
\begin{proof}
For the $\widetilde \Lambda_t^{\rm robust}$ matrix (i.e., the third part in Eq.~\eqref{eq:covariance-rorl}), we denote the covariance matrix for a specific $i$ as $\Phi^i_t$. Then we have $\widetilde \Lambda_t^{\rm{ood\_diff}}=\sum_{i=1}^{m} \Phi^i_t$. In the following, we discuss the condition of positive-definiteness of $\Phi^i_t$. For the simplicity of notation, we omit the superscript and subscript of $s_t^i$ and $a_t^i$ for given $i$ and $t$. Specifically, we define
\begin{equation}\nonumber
\Phi^i_t = \frac{1}{|\mathbb{B}_d(s_t^i,\epsilon)|}  \sum_{\hat s_j \sim \mathcal{D}_{\text{ood}}(s)} \big[\phi(\hat s_j,a) - \phi(s,a)\big]\big[\phi(\hat s_j,a) - \phi(s,a)\big]^\top,
\end{equation}
where $j\in\{1,\ldots,N\}$ indicates we sample $|\mathbb{B}_d(s_t^i,\epsilon)|=N$ perturbed states for each $s$. For a nonzero vector $y\in \mathbb{R}^d$, we have
\begin{equation}\label{eq:app-semi-definite}
\begin{aligned}
y^\top \Phi^i_t y &= y^\top\left(\frac{1}{N} \sum_{j=1}^N \big(\phi(\hat s_j,a)-\phi(s,a)\big)\big(\phi(\hat s_j,a)-\phi(s,a)\big)^\top\right) y \\
&= \frac{1}{N} \sum_{j=1}^N y^\top \big(\phi(\hat s_j,a)-\phi(s,a)\big)\big(\phi(\hat s_j,a)-\phi(s,a)\big)^\top y \\
&=\frac{1}{N} \sum_{j=1}^N \left(\big(\phi(\hat s_j,a)-\phi(s,a)\big)^\top y \right)^2 \geq 0,
\end{aligned}
\end{equation}
where the last inequality follows from the observation that $\big(\phi(\hat s_j,a)-\phi(s,a)\big)^\top y$ is a scalar. Then $\Phi^i_t$ is always positive semi-definite. In the following, we denote $z_j=\phi(\hat s_j,a)-\phi(s,a)$. Then we need to prove that the condition to make $\Phi^i_t$ positive definite is ${\rm rank}[z_1,\ldots,z_N]=d$, where $d$ is the feature dimension. Our proof follows contradiction. 

In Equation~\eqref{eq:app-semi-definite}, when $y^\top \Phi^i_t y=0$ with a nonzero vector $y$, we have $z_j^\top y=0$ for all $j=1,\ldots,N$. Suppose the set $\{z_1,\ldots,z_N\}$ spans $\mathbb{R}^d$, then there exist real numbers $\{\alpha_1,\ldots,\alpha_N\}$ such that $y=\alpha_1  z_1 +\dots+\alpha_N z_N$. But we have $y^\top y=\alpha_1  z_1^\top y + \dots +\alpha_N z_N^\top y=\alpha_1 \times 0+\ldots+\alpha_N \times 0=0$, yielding that $y=\mathbf{0}$, which forms a contradiction. Hence, if the set $\{z_1,\ldots,z_N\}$ spans $\mathbb{R}^d$, which is equivalent to ${\rm rank}[z_1,\ldots,z_N]=d$, then $\Phi^i_t$ is positive definite.

Under the given conditions, since the size of samples $\mathbb{B}_d(s_t^i,\epsilon)$ is sufficient and the neural network maintains useful variability to make the Jacobian matrix of $\phi(s,a)$ have full rank, it ensures that $\exists k\in[1,m]$, for any nonzero vector $y\in \mathbb{R}^d$, $y^\top \Phi^k_t y > 0$. We have $y^\top \widetilde \Lambda_t^{\rm robust} y = \sum_{i=1}^{m} y^\top \Phi^i_t y \geq y^\top \Phi^k_t y > 0$. Therefore, $\widetilde \Lambda_t^{\rm robust}$ is positive definite, which concludes our proof.
\end{proof}

Recall the covariance matrix of the variant algorithm without smoothness constraint is $\widetilde \Lambda_t^{\text{variant}}=\widetilde \Lambda_t^{\text{in}}+\widetilde \Lambda_t^{\text{ood}}$, and RO2O has a covariance matrix as $\widetilde \Lambda_t=\widetilde \Lambda_t^{\rm variant}+\widetilde \Lambda_t^{\rm robust}$, we have the following corollary based on Lemma \ref{tm:PD_matrix}.

\setcounter{theorem}{0}
\begin{theorem}[restate]
Assuming that the size of adversarial samples $\mathbb{B}_d(s^i_t,\epsilon)$ is sufficient and the Jacobian matrix of $\phi(s, a)$
has full rank, the smoothness constraint leads to smaller uncertainty for $\forall (s^{\star},a^{\star})\in \cS\times\cA$, as 
\begin{equation}
\nonumber
\Gamma^{\rm lcb}(s^{\star},a^{\star};\cD_{\rm RO2O}) < \Gamma^{\rm lcb}(s^{\star},a^{\star};\cD_{\rm variant}),
\end{equation}
where the covariance matrices for these two LCB terms are $\widetilde \Lambda_t$ in Equation~\eqref{eq:covariance-rorl2} and $\widetilde \Lambda^{\rm in}_t+\widetilde \Lambda^{\rm ood}_t$, respectively.
\label{thm:distribution_shift2}
\end{theorem}

\begin{proof}
According to Lemma 1, since $\Lambda_t^{\rm robust}$ is positive-definite, we have $\Lambda_t^{\rm robust}\succeq \lambda I$ with a factor $\lambda > 0$. Meanwhile, the factor $\lambda$ can be large if we have sufficient adversarial samples and also with large variability in adversarial samples. By assuming $\widetilde \Lambda_t^{\rm in}+\widetilde \Lambda_t^{\rm ood}$ is positive definite and leveraging the properties of generalized Rayleigh quotient, we have
\begin{equation}
\nonumber
\begin{aligned}
\frac{\phi^\top (\widetilde \Lambda_t^{\rm variant})^{-1} \phi}{\phi^\top (\widetilde \Lambda_t^{\rm variant}+\widetilde \Lambda_t^{\rm robust})^{-1} \phi} &\geq \lambda_{\text{min}}\big((\widetilde \Lambda_t^{\rm variant}+\widetilde \Lambda_t^{\rm robust})(\widetilde \Lambda_t^{\rm variant})^{-1}\big) \\
&= 
\lambda_{\text{min}}\big( \mathrm{I} +(\widetilde \Lambda_t^{\rm robust})(\widetilde \Lambda_t^{\rm variant})^{-1}\big) \\&= 1 + \lambda_{\text{min}}\big((\widetilde \Lambda_t^{\rm robust})(\widetilde \Lambda_t^{\rm variant})^{-1}\big).
\end{aligned}
\end{equation}
Since $\widetilde \Lambda_t^{\rm robust}$ and $(\widetilde \Lambda_t^{\rm variant})^{-1}$ are both positive definite, the eigenvalues of $\widetilde \Lambda_t^{\rm robust}(\widetilde \Lambda_t^{\rm variant})^{-1}$ are all positive: $\lambda_{\text{min}}\big(\widetilde \Lambda_t^{\rm robust}(\widetilde \Lambda_t^{\rm variant})^{-1}\big) > 0$,
where $\lambda_{\rm min}(\cdot)$ is the minimum eigenvalue. 

Recall the uncertainty is calculated as $\Gamma^{\rm lcb}(s_t,a_t;\cD)= \beta_t \big[\phi(s_t,a_t)^\top\Lambda_t^{-1}\phi(s_t,a_t)\big]^{\nicefrac{1}{2}}$. Then for $\forall \phi(s^\star,a^\star)$, we have 
\begin{equation}
\begin{aligned}
\nonumber
\phi(s^\star,a^\star)^\top (\widetilde \Lambda_t^{\rm variant})^{-1} \phi(s^\star,a^\star)& >\phi(s^\star,a^\star)^\top (\widetilde \Lambda_t^{\rm variant}+\widetilde \Lambda_t^{\rm robust})^{-1} \phi(s^\star,a^\star)\\&=\phi(s^\star,a^\star)^\top (\widetilde \Lambda_t)^{-1} \phi(s^\star,a^\star),
\end{aligned}
\end{equation}
which means that 
$\Gamma^{\rm lcb}(s^{\star},a^{\star};\cD_{\rm variant})>\Gamma^{\rm lcb}(s^{\star},a^{\star};\cD_{\rm RO2O})$ and concludes our proof.
\end{proof}

As an extreme case in tabular MDPs where the states and actions are finite, the LCB-penalty takes a simpler form. Specifically, 
we consider the joint state-action space $D=|\mathcal{S}|\times|\mathcal{A}|$. Then $j$-th state-action pair can be encoded as a one-hot vector as $\phi(s,a)\in \mathbb{R}^{D}$, where $j\in [0,D-1]$. By considering the tabular MDP as a special case of the linear MDP \shortcite{yang2019sample,lsvi-2020}, we define
\begin{equation}
\phi(s_j,a_j)=
\begin{bmatrix}
\begin{smallmatrix}
0\\
\vdots\\
1\\
\vdots\\
0
\end{smallmatrix}
\end{bmatrix}
\in \mathbb{R}^D,
\qquad
\phi(s_j,a_j)\phi(s_j,a_j)^{\top}=
\begin{bmatrix}
\nonumber
\begin{smallmatrix}
0      & \cdots & 0 & \cdots & 0\\
\vdots & \ddots &   &        & \vdots \\
0      &        & 1 &        & 0\\
\vdots &        &   & \ddots & \vdots \\
0      & \cdots & 0 & \cdots & 0\\
\end{smallmatrix}
\end{bmatrix}
\in \mathbb{R}^{D\times D},
\end{equation}
where the value of $\phi(s_j,a_j)$ is $1$ at the $j$-th entry and $0$ elsewhere. Then the matrix $\Lambda_j=\sum_{i=0}^{m}\phi(s_j^i,a_j^i)\phi(s_j^{i},a_j^{i})^\top$ is a specific covariance matrix based on the learned datasets. It takes the form of
\begin{equation}
\Lambda_j=
\begin{bmatrix}
\nonumber
\begin{smallmatrix}
n_0    & 0      &        & \cdots              & 0      \\
0      & n_1    &        & \cdots              & 0      \\
\vdots &        & \ddots &                     & \vdots \\
0      & \cdots &        & \!\!\!\!\!\!\!\!n_j & 0      \\
\vdots &        &        & \:\:\:\:\:\:\ddots  & \vdots \\
0      & \cdots &        & \cdots              & n_{d-1}
\end{smallmatrix}
\end{bmatrix},
\end{equation}
where the $j$-th diagonal element of $\Lambda_j$ is the corresponding counts for state-action $(s_j,a_j)$, i.e., 
\[n_j=N_{s_j,a_j}.\] 
It thus holds that 
\begin{equation}\label{app::count}
\big[\phi(s_j,a_j)^{\top}\Lambda_j^{-1}\phi(s_j,a_j)\big]^{\nicefrac{1}{2}}=\frac{1}{\sqrt{N_{s_j,a_j}}}.
\end{equation}

For the variant algorithm of RO2O in Equation~\eqref{eq:app_variant}, since the value function is learned from $\cD_{\rm variant}$, the counting function also counts from this dataset. However, without any constraints, the count for a purely OOD state-action pair $(s^{\star},a^{\star})$ can approach zero, and thus $\Gamma^{\rm lcb}(s^{\star},a^{\star};\cD_{\rm variant})\to \infty$ according to Equation~\eqref{app::count}. In contrast, as we proved in Lemma~\ref{tm:PD_matrix}, the covariance matrix $\widetilde \Lambda_t^{\rm robust}$ for smoothness constraints is positive-definite as $\widetilde \Lambda_t^{\rm robust} \succeq \lambda \cdot \mathrm{I}$ where $\lambda > 0$. Then the covariance matrix for RO2O as $\widetilde \Lambda_t \succeq \lambda \cdot \mathrm{I}$ since $\widetilde \Lambda_t=\widetilde \Lambda_t^{\rm variant}+\widetilde \Lambda_t^{\rm robust}$. Then, we have 
$\big[\phi(s_j,a_j)^{\top}\Lambda_j^{-1}\phi(s_j,a_j)\big]^{\nicefrac{1}{2}} < \nicefrac{1}{\sqrt{\lambda}}$ and thus $\Gamma^{\rm lcb}(s^{\star},a^{\star};\cD_{\rm RO2O})\leq \nicefrac{\beta_t}{\sqrt{\lambda}}$ with $\lambda>0$. As a result, RO2O is more robust to significant distribution shift theoretically. 

\subsection{Sub-optimality Gap of RO2O}

To quantify the sub-optimality gap, we start by the following lemma to show the ensemble $Q$-networks used in RO2O can recover the LCB term in linear MDPs.

\begin{lemma}[Equivalence between LCB-penalty and Ensemble Uncertainty] We assume that the noise in linear regression follows the standard Gaussian, then it holds for the posterior of $w$ given $\cD_i$ that
\begin{equation}
\nonumber
\VV_{\hw}[Q_i(s,a)]=\VV_{\hw}\bigl(\phi(s, a)^\top \hw \bigr) = \phi(s, a)^\top \Lambda^{-1} \phi(s, a), \quad \forall (s, a)\in\cS\times\cA.
\end{equation}
\end{lemma}
\begin{proof}
We refer to the proof in Lemma 1 of \shortcite{pbrl}. 
\end{proof}

In RO2O, we choose the minimum value among ensemble $Q$-networks (i.e., $\min Q_i$) as the learning target, which is equivalent to the uncertainty penalty as $i.e., \bar{Q}-\alpha \sqrt{\VV (Q_i)}$ with a specific $\alpha$ \shortcite{SAC-N}. The following theorem shows RO2O can consistently reduce the sub-optimality gap with online fine-tuning.

\begin{theorem}
\label{the:the2}
Under the same conditions as Theorem~\ref{thm:distribution_shift2}, with additional online experience in the fine-tuning phase, the sub-optimality gap holds for RO2O in linear MDPs, as
\begin{equation}
\begin{aligned}
\label{eq:app_opt_gap}
{\rm SubOpt} (\pi^*, \widetilde\pi) &\leq \sum\nolimits_{t=1}^{T} \mathbb{E}_{\pi^*} \big[ \Gamma^{\rm lcb}(s_t,a_t;\cD_{\rm RO2O}\cup 
\cD_{
\rm online
}) \big] \\
&\leq \sum\nolimits_{t=1}^{T} \mathbb{E}_{\pi^*} \big[ \Gamma^{\rm lcb}(s_t,a_t;\cD_{\rm RO2O}) \big],
\end{aligned}
\end{equation}
where $\widetilde\pi$ and $\pi^{*}$ are the learned policy and the optimal policy in $\cD_{\rm RO2O}\cup 
\cD_{
\rm online
}$, respectively.
\label{thm:optimality}
\end{theorem}

\begin{proof}
Based on the LSVI solution of $\widetilde w_{\rm offline}$, we consider importing additional dataset $\cD_{\rm finetune}$ in online interactions. Following a similar solution procedure as in 
Equation~\eqref{eq:app_simplified_problem} via LSVI, we obtain the solution of RO2O with online dataset as 
\begin{equation}
\nonumber
\widetilde w^{\rm R2O2}_t =(\widetilde{ \Lambda}_t^{\rm RO2O})^{-1} \Big( \sum_{(s,a,y)\sim \cD_{\rm offline}\cup \cD_{\rm finetune}} \phi(s,a) y + \sum_{(\hat s, \hat a, \hat y) \sim \widehat{\mathcal{D}}_{\text{ood}} }\phi(\hat s,\hat a)  \hat y  \Big), 
\end{equation}
where $\widehat{\mathcal{D}}_{\rm ood}$ is a new OOD dataset that contains OOD samples of both the offline and online data. The new covariance matrix $\widetilde{\Lambda}^{\rm RO2O}_t$ is calculated on samples in both online and offline data,
\begin{equation}
\begin{aligned}
\label{eq:app_covariance-rorl-online}
\widetilde{\Lambda}^{\rm RO2O}_t &=\widetilde \Lambda_t + \widetilde \Lambda^{\rm online}_t\\&=\sum_{\cD_{\rm offline}\cup \cD_{\rm finetune}} \phi(s_t^i,a_t^i)\phi(s_t^i,a_t^i)^\top+\sum_{\widehat{\cD}_{\text{ood}}} \phi(\hat s_t,\hat a_t)\phi(\hat s_t,\hat a_t)^\top \\& + 
\sum_{i=1}^{m} \frac{1}{|\mathbb{B}_d(s_t^i,\epsilon)|}  \sum_{\widehat{\cD}_{\text{robust}}} \big[\phi(\hat s_t^i,a_t^i) - \phi(s_t^i,a_t^i)\big]\big[\phi(\hat s_t^i,a_t^i) - \phi(s_t^i,a_t^i)\big]^\top,
\end{aligned}
\end{equation}
where each term is calculated on both the offline dataset and online fine-tuning dataset since the proposed RO2O algorithm does not change the learning objective in the offline-to-online process. We denote the total data used in online fine-tuning as $\cD_{\rm online}$, which contains the $\cD_{\rm finetune}$ collected in interacting with the environment, the additional adversarial samples, and the OOD samples that are constructed based on $\cD_{\rm finetune}$. Then, $\widetilde{\Lambda}^{\rm RO2O}_t$ is the covariance matrix of samples from both offline and online datasets, i.e., $\cD_{\rm RO2O}\cup \cD_{\rm online}$.

According to the theoretical framework of pessimistic value-iteration \shortcite{jin2021pessimism}, value iteration with LCB-based uncertainty penalty is provable efficient in offline RL. Based on the covariance matrix of RO2O, the LCB-term 
of RO2O learning in offline pre-training and online-fine-tuning are
\begin{align}
\Gamma^{\rm lcb}(s_t,a_t;\cD_{\rm RO2O})&=\beta_t [\phi(s_t, a_t)^\top (\widetilde{\Lambda}_t)^{-1} \phi(s_t, a_t)]^{\nicefrac{1}{2}},
\label{eq:app_lcb_offline}
\\ \quad 
{\rm and\quad}\Gamma^{\rm lcb}(s_t,a_t;\cD_{\rm RO2O}\cup 
\cD_{
\rm online
})&=\beta_t [\phi(s_t, a_t)^\top (\widetilde{\Lambda}^{\rm RO2O}_t)^{-1} \phi(s_t, a_t)]^{\nicefrac{1}{2}},
\label{eq:app_lcb_total}
\end{align}
respectively, where $\widetilde{\Lambda}^{\rm RO2O}_t$ is defined in Equation~\eqref{eq:app_covariance-rorl-online}. According to the definition of $\xi$-uncertainty quantifier \shortcite{lsvi-2020}, $\Gamma^{\rm lcb}(s_t,a_t;\cD_{\rm RO2O}\cup 
\cD_{
\rm online
})$ also forms a valid $\xi$-uncertainty quantifier under mild assumptions \shortcite{rorl}. According to \shortcite{jin2021pessimism}, since $\Gamma^{\rm lcb}(s_t,a_t;\cD_{\rm RO2O})$ is a valid $\xi$-uncertainty quantifier, the first inequality of Equation~\eqref{eq:app_opt_gap} holds in quantifying the sub-optimality gap. Further, since $\widetilde{\Lambda}^{\rm RO2O}_t \succeq \widetilde{\Lambda}_t$ according to Equation~\ref{eq:app_covariance-rorl-online}, we have $\Gamma^{\rm lcb}(s_t,a_t;\cD_{\rm RO2O}\cup 
\cD_{
\rm online
}) \leq \Gamma^{\rm lcb}(s_t,a_t;\cD_{\rm RO2O})$ by following Equation \ref{eq:app_lcb_total} and \ref{eq:app_lcb_offline}, which concludes our proof.
\end{proof}

\begin{definition}[$\xi$-Uncertainty Quantifier]
The set of penalization $\{\Gamma_t\}_{t\in[T]}$ forms a $\xi$-Uncertainty Quantifier if it holds with probability at least $1 - \xi$ that
\begin{equation*}
|\widehat \cT V_{t+1}(s, a) - \cT V_{t+1}(s, a)| \leq \Gamma_t(s, a)
\end{equation*}
for all $(s, a)\in\cS\times\cA$, where $\cT$ is the Bellman equation and $\widehat \cT$ is the empirical Bellman equation that estimates $\cT$ based on the offline data.
\end{definition}

Following PBRL \shortcite{pbrl} and RORL \shortcite{rorl} that adopt ensemble disagreement as the uncertainty quantifier, in linear MDPs, the proposed ensemble uncertainty $\beta_t\cdot\mathcal{U}(s_t,a_t)$ is an estimation to the LCB-penalty $\Gamma^{\rm lcb}(s_t,a_t)$ for an appropriately selected tuning parameter $\beta_t$. As a result, our method enjoys a similar form of optimality gap in PEVI.

Further, since our method adopts additional OOD sampling and smooth constraints, the covariance matrix in calculating $\Gamma^{\rm lcb}(s_t,a_t;\mathcal{D}_{\rm RO2O})$ for our method becomes
\begin{equation*}
\Gamma^{\rm lcb}(s_t,a_t;\mathcal{D}_{\rm RO2O})=\beta_t [\phi(s_t, a_t)^\top (\widetilde{\Lambda}_t)^{-1} \phi(s_t, a_t)]^{\frac{1}{2}},
\end{equation*}
where
\begin{equation*}
\tilde{\Lambda}_t=\tilde{\Lambda}_t^{\rm in}+\tilde{\Lambda}_t^{\rm ood}+\tilde{\Lambda}_t^{\rm robust},
\end{equation*}
which also serves as a $\xi$-uncertainty quantifier. Then the uncertainty term for RO2O is $\Gamma_{i}^{\rm lcb}(s_t,a_t;\mathcal{D}_{\rm RO2O})$ in offline setting, and in online exploration it becomes $\Gamma_{i}^{\rm LCB}(s_t,a_t;\mathcal{D}_{\rm RO2O}\cup \mathcal{D}_{\rm online})$. Further, our theoretical analysis in Theorem \ref{the:the2} shows that
\begin{equation*}
{\rm SubOpt} (\pi^*, \widetilde\pi) \leq \sum\nolimits_{t=1}^{T} \mathbb{E}_{\pi^*} \big[ \Gamma_i^{\rm lcb}(s_t,a_t;\mathcal{D}_{\rm RO2O}\cup 
\mathcal{D}_{
\rm online
}) \big] \leq \sum\nolimits_{t=1}^{T} \mathbb{E}_{\pi^*} \big[ \Gamma_i^{\rm lcb}(s_t,a_t;\mathcal{D}_{\rm RO2O}) \big],
\end{equation*}
which signifies the online exploration data can consistently reduce the sub-optimality gap of our method with $\xi$-uncertainty quantification.

\section{Environmental Settings}
\label{app:environment}
In this section, we introduce more details of the experimental environments. 
\paragraph{MuJoCo Locomotion}
We conduct experiments on three MuJoCo locomotion tasks from D4RL\shortcite{d4rl}, namely HalfCheetah, Walker2d, and Hopper. The goal of each task is to move forward as far as possible without falling, while keeping the control cost minimal. For each task, we consider three types of datasets. The medium datasets contain trajectories collected by medium-level policies. The medium-replay datasets encompass all samples collected during the training of a medium-level agent from scratch. In the case of the medium-expert datasets, half of the data comprises rollouts from medium-level policies, while the other half consists of rollouts from expert-level policies. In this study, we exclude the random datasets, as in typical real-world scenarios, we rarely use a random policy for system control. We utilize the v2 version of each dataset. For offline phase, We train agents for 2.5M gradient steps over all datasets with an ensemble size of $N = 10$. Then we run online fine-tuning for an additional 250K environment interactions.
\paragraph{Antmaze Navigation}
We also evaluate our method on the Antmaze navigation tasks that involve controlling an 8-DoF ant quadruped robot to navigate through mazes and reach a desired goal. The agent receives binary rewards based on whether it successfully reaches the goal or not. We study each method using the following datasets from D4RL \shortcite{d4rl}: large-diverse, large-play, medium-diverse, medium-play, umaze-diverse, and umaze. The difference between diverse and play datasets is the optimality of the trajectories they contain. The diverse datasets consist of trajectories directed towards random goals from random starting points, whereas the play datasets comprise trajectories directed towards specific locations that may not necessarily correspond to the goal. We use the v2 version of each dataset. For offline phase, We train agents for 1M gradient steps over all datasets with an ensemble size of $N = 10$. Then we run online fine-tuning for an additional 250K environment interactions.

\section{Implementation Details}
\label{app:implement}
In this section, we introduce implementation details and hyper-parameters for each task.

\paragraph{MuJoCo Locomotion}
We select PEX, AWAC, SAC, SPOT, Cal-QL and IQL as our baselines in mujoco locomotion tasks. For SAC, AWAC, SPOT, Cal-QL and IQL, we use the implementation from CORL\footnote{https://github.com/tinkoff-ai/CORL} with default hyperparameters. For PEX, we use the open-source code of the original paper\footnote{https://github.com/Haichao-Zhang/PEX}. 
To compare the fine-tuning performance of the algorithms under limited online interactions, we reduce the number of online interaction steps from the previous 1M to 250K. All the hyper-parameters used in RO2O for the benchmark experiments are listed in Table~\ref{Hyper in D4RL}. $\eta_1$, $\eta_2$, $\eta_3$ indicate the coefficient of the $Q$-network smoothing loss $\mathcal{L}_{\text{Qsmooth}}$, ood loss $\mathcal{L}_{\text{ood}}$ and the policy smoothing loss $\mathcal{L}_{\text{policy}}$, respectively, where $\eta_1$ maintains a constant value of 0.0001, $\eta_2$ is tuned within \{0.0, 0.1, 0.5\} and $\eta_3$ is searched in \{0.1, 1.0\}. Additionally, for the above three losses, we construct a perturbation set $\mathbb{B}_d(s,\epsilon)=\{\hat{s}:d(s,\hat{s})\leq \epsilon\}$ by setting different epsilons $\epsilon$. We denote the perturbation scales for the Q value functions, the policy, and the OOD loss as $\epsilon_Q$, $\epsilon_P$, $\epsilon_{ood}$. $\tau$ is set to control the weight of $\mathcal{L}_{\text{Qsmooth}}$ which maintains a constant value of 0.2. The number of sampled perturbed observations n is set for tuning within \{10, 20\}. And $\alpha$ is set to control the pessimistic degree of $\mathcal{L}_{\text{ood}}$ during the pre-trained phase. Moreover, discarding offline data buffer is adopted in RO2O, which exhibits benefits for stable transfer in our experiments and mitigates the distributional shift.

\begin{table*}[ht]
\centering
\tiny
\caption{Hyperparameters of RO2O for the MuJoCo domains.}
\label{Hyper in D4RL}
\resizebox{\textwidth}{!}{
\begin{tabular}{l|c|c|l|c|c|c|c|c|c}
\cline{1-10}
\textbf{Task Name} & $\eta_1$ & $\eta_2$ & \multicolumn{1}{c|}{$\eta_3$} & $\epsilon_{\mathrm{Q}}$ & $\epsilon_{\mathrm{P}}$ & $\epsilon_{\text {ood}}$ & $\tau$ & $n$ & $\alpha$ \\ \cline{1-10}
halfcheetah-medium                  & \multirow{4}{*}{0.0001} & \multirow{4}{*}{0.0} & \multirow{4}{*}{0.1}          & 0.001                   & 0.001                   & \multirow{4}{*}{0.00}    & \multirow{4}{*}{0.2} & 10                  & \multirow{4}{*}{0}                \\
halfcheetah-medium-replay           &                         &                      &                               & 0.001                   & 0.001                   &                          &                      & 10                  &                                   \\
halfcheetah-medium-expert           &                         &                      &                               & 0.001                   & 0.001                   &                          &                      & 10                  &                                   \\
halfcheetah-expert                  &                         &                      &                               & 0.005                   & 0.005                   &                          &                      & 10                  &                                   \\ \hline
hopper-medium                       & \multirow{4}{*}{0.0001} & \multirow{4}{*}{0.5} & \multirow{4}{*}{0.1}          & \multirow{4}{*}{0.005}  & \multirow{4}{*}{0.005}  & \multirow{4}{*}{0.01}    & \multirow{4}{*}{0.2} & \multirow{4}{*}{20} & 2.0 $\rightarrow$ 0.1 ($1e^{-6}$) \\
hopper-medium-replay                &                         &                      &                               &                         &                         &                          &                      &                     & 0.1 $\rightarrow$ 0.0 ($1e^{-6}$) \\
hopper-medium-expert                &                         &                      &                               &                         &                         &                          &                      &                     & 3.0 $\rightarrow$ 1.0 ($1e^{-6}$) \\
hopper-expert                       &                         &                      &                               &                         &                         &                          &                      &                     & 4.0 $\rightarrow$ 1.0 ($1e^{-6}$) \\ \hline
walker2d-medium                     & \multirow{4}{*}{0.0001} & 0.1                  & \multirow{4}{*}{1.0}          & 0.01                    & 0.01                    & \multirow{4}{*}{0.01}    & \multirow{4}{*}{0.2} & \multirow{4}{*}{20} & 1.0 $\rightarrow$ 0.1 ($5e^{-7}$) \\
walker2d-medium-replay              &                         & 0.1                  &                               & 0.01                    & 0.01                    &                          &                      &                     & 0.1 $\rightarrow$ 0.1 (0.0)       \\
walker2d-medium-expert              &                         & 0.1                  &                               & 0.01                    & 0.01                    &                          &                      &                     & 0.1 $\rightarrow$ 0.1 (0.0)       \\
walker2d-expert                     &                         & 0.5                  &                               & 0.005                   & 0.005                   &                          &                      &                     & 1.0 $\rightarrow$ 0.5 ($1e^{-6}$) \\ \hline
\end{tabular}}
\end{table*}


\paragraph{Antmaze Navigation}
We select PEX, SPOT and Cal-QL as our baselines in antmaze navigation tasks. For SPOT and Cal-QL, we use the implementation provided by CORL with default hyperparameters. We directly used the experimental results provided by CORL in weight \& bias for comparison. For PEX, we use the open-source code of the original paper. To compare the fine-tuning performance of the algorithms under limited online interactions, we reduce the number of online interaction steps from the previous 1M to 250K. We found that incorporating behavior cloning (BC) during the offline pre-training phase of the AntMaze task can effectively improve model performance. Additionally, making appropriate adjustments to BC during the online fine-tuning phase for certain tasks can also enhance the algorithm's performance and stability. And we transform AntMaze rewards according to $4(r - 0.5)$ as per MSG\shortcite{msg} or CQL\shortcite{cql}. All the hyper-parameters used in RO2O for the benchmark experiments are listed in Table \ref{Hyper in antmaze}. $\beta_{\text{BC, off}}$ and $\beta_{\text{BC, on}}$ indicate the weight of BC regularization on policy networks during offline pre-training and online fine-tuning, respectively. The LCB policy objective and `Min' policy objective represent optimizing the policy network using $\rm Mean(Q_{\theta_i}(s,a))-\rm Std(Q_{\theta_i}(s,a))$ or $\min\limits_{i} Q_{\theta_i}(s,a)$, respectively. And the meanings of other parameters remain consistent with Table~\ref{Hyper in D4RL} under the Mujoco tasks.

\begin{table*}[ht]
\centering
\tiny
\caption{Hyper-parameters of RO2O for the AntMaze domains.}
\label{Hyper in antmaze}
\resizebox{\linewidth}{!}{
\begin{tabular}{l|c|c|c|c|c|c|c|c|c}
\hline 
\textbf{Task Name} & $\eta_2$ & $\eta_3$ & $\epsilon_{\mathrm{P}}$ & $\epsilon_{\text {ood}}$ & $n$ & policy objective & $\beta_{\text{BC, off}}$ & $\beta_{\text{BC, on}}$ & $\alpha$ \\
\hline 
umaze & & 0.3 & & & & LCB & 5 & 5 & 1.0 $\rightarrow$ 1.0 (0.0) \\
umaze-diverse & & 0.3 & & & & LCB & 10 & 20 & 2.0 $\rightarrow$ 2.0 (0.0) \\
medium-play & \multirow{2}{*}{1.0} & 0.3 & \multirow{2}{*}{0.005} & \multirow{2}{*}{0.01} & \multirow{2}{*}{20} & LCB & 2 & 2 & 1.0 $\rightarrow$ 1.0 (0.0) \\
medium-diverse & & 0.3 & & & & LCB & 4 & 4 & 2.0 $\rightarrow$ 1.0 ($1e^{-6}$) \\
large-play & & 0.5 & & & & Min & 2 & 8 & 2.0 $\rightarrow$ 1.0 ($1e^{-6}$) \\
large-diverse & & 0.3 & & & & Min & 2 & 8 & 1.0 $\rightarrow$ 1.0 (0.0)\\
\hline
\end{tabular}}
\end{table*}

\section{More Discussion}
\label{more discussion}
\paragraph{Using different learning target for different tasks} Most of the ensemble-based RL algorithms use shared pessimistic target values when computing each ensemble member’s Bellman error. However, the results reported in the reference \shortcite{msg,rorl} and our experiments demonstrate that using independent target surpasses highly well-tuned state-of-the-art methods by a wide margin on challenging domains such as Antmaze. We believe there are several reasons: (i) Antmaze navigation tasks are more complex than Mujoco locomotion tasks. Since there is significant distribution shift between online interactions and offline data, it will be challenging to learn effective policies by relying solely on policies derived from offline data. Due to shared TD target is too pessimistic, agents tend to avoid accessing a significant number of ODD samples, thereby limiting exploration to some extent. This also results in methods like PBRL \shortcite{pbrl} and EDAC \shortcite{SAC-N}, which utilize shared TD targets, performing poorly on tasks such as Antmaze. (ii) In contrast, the disparity between in-distribution and OOD policies is not obvious in Mujoco tasks. Therefore, directly applying shared targets to achieve pessimistic updates in Mujoco tasks ensures pessimism while also capturing the uncertainty near the distribution. Therefore, we refer to the independent target used in the Q-value Bellman update by MSG \shortcite{msg}.

\paragraph{Comparison to RORL}
Here, we discuss the differences between RORL \shortcite{rorl} and our method from several perspectives. (i) \textbf{Motivation.} The motivation of robustness constraints in RORL is to improve the smoothness of policy and Q-functions in facing adversarial attacks. In contrast, our method focuses on offline-to-online settings, where robustness regularization is used to prevent the distribution shift of OOD data in online exploration. We highlight that both RORL and our method adopt the same smooth value function/policy originally proposed in online exploration, while the motivations for introducing robustness in our method and RORL are quite different. (ii) From a \textbf{theoretical} perspective, we provide new theoretical results in Theorem \ref{thm:optimality}, which analyzes the optimality gap of the learned policy in online exploration with additional online datasets. With uncertainty quantification and smoothness constraints, our method benefits from more online exploration data without suffering from distribution shifts, which is crucial for offline-to-online RL. Our theoretical result shows the optimality gap of our method shrinks if the online exploration data increases the data coverage of the optimal policy, which is significantly different from previous offline-to-online methods that should penalize OOD data in online exploration. (iii) \textbf{Empirically}, our method obtains strong performance without a specially designed online adaptation process. The offline-to-online performance does not drop when interacting with the online environment, which is consistent with our theoretical analysis. Benefiting from the theoretical result, our method can perform efficient policy improvement without specific modifications to the learning architecture in the offline-to-online process.

\bibliography{RO2O}

\begin{thebibliography}{}

\bibitem[\protect\BCAY{An, Moon, Kim,\ \BBA\ Song}{An et~al.}{2021}]{SAC-N}
An, G., Moon, S., Kim, J.-H., \BBA\ Song, H.~O. \BBOP2021\BBCP.
\newblock \BBOQ Uncertainty-based offline reinforcement learning with diversified {Q}-ensemble\BBCQ\
\newblock {\Bem Advances in neural information processing systems}, {\Bem 34}, 7436--7447.

\bibitem[\protect\BCAY{Bai, Wang, Yang, Deng, Garg, Liu,\ \BBA\ Wang}{Bai et~al.}{2022}]{pbrl}
Bai, C., Wang, L., Yang, Z., Deng, Z.-H., Garg, A., Liu, P., \BBA\ Wang, Z. \BBOP2022\BBCP.
\newblock \BBOQ Pessimistic bootstrapping for uncertainty-driven offline reinforcement learning\BBCQ\
\newblock In {\Bem International Conference on Learning Representations}.

\bibitem[\protect\BCAY{Berner, Brockman, Chan, Cheung, Debiak, Dennison, Farhi, Fischer, Hashme, Hesse, J{\'{o}}zefowicz, Gray, Olsson, Pachocki, Petrov, de~Oliveira~Pinto, Raiman, Salimans, Schlatter, Schneider, Sidor, Sutskever, Tang, Wolski,\ \BBA\ Zhang}{Berner et~al.}{2019}]{dota}
Berner, C., Brockman, G., Chan, B., Cheung, V., Debiak, P., Dennison, C., Farhi, D., Fischer, Q., Hashme, S., Hesse, C., J{\'{o}}zefowicz, R., Gray, S., Olsson, C., Pachocki, J., Petrov, M., de~Oliveira~Pinto, H.~P., Raiman, J., Salimans, T., Schlatter, J., Schneider, J., Sidor, S., Sutskever, I., Tang, J., Wolski, F., \BBA\ Zhang, S. \BBOP2019\BBCP.
\newblock \BBOQ Dota 2 with large scale deep reinforcement learning\BBCQ\
\newblock {\Bem CoRR}, {\Bem abs/1912.06680}.

\bibitem[\protect\BCAY{Chen, Sidor, Abbeel,\ \BBA\ Schulman}{Chen et~al.}{2017}]{ucbexploration}
Chen, R.~Y., Sidor, S., Abbeel, P., \BBA\ Schulman, J. \BBOP2017\BBCP.
\newblock \BBOQ Ucb exploration via {Q}-ensembles\BBCQ\
\newblock {\Bem CoRR}, {\Bem abs/1706.01502}.

\bibitem[\protect\BCAY{Chen, Wang, Zhou,\ \BBA\ Ross}{Chen et~al.}{2021}]{redq}
Chen, X., Wang, C., Zhou, Z., \BBA\ Ross, K.~W. \BBOP2021\BBCP.
\newblock \BBOQ Randomized ensembled double {Q}-learning: Learning fast without a model\BBCQ\
\newblock In {\Bem International Conference on Learning Representations}.

\bibitem[\protect\BCAY{Fu, Kumar, Nachum, Tucker,\ \BBA\ Levine}{Fu et~al.}{2020}]{d4rl}
Fu, J., Kumar, A., Nachum, O., Tucker, G., \BBA\ Levine, S. \BBOP2020\BBCP.
\newblock \BBOQ {D4RL:} datasets for deep data-driven reinforcement learning\BBCQ\
\newblock {\Bem CoRR}, {\Bem abs/2004.07219}.

\bibitem[\protect\BCAY{Fujimoto\ \BBA\ Gu}{Fujimoto\ \BBA\ Gu}{2021}]{td3bc}
Fujimoto, S.\BBACOMMA\  \BBA\ Gu, S.~S. \BBOP2021\BBCP.
\newblock \BBOQ A minimalist approach to offline reinforcement learning\BBCQ\
\newblock {\Bem Advances in neural information processing systems}, {\Bem 34}, 20132--20145.

\bibitem[\protect\BCAY{Fujimoto, Hoof,\ \BBA\ Meger}{Fujimoto et~al.}{2018}]{td3}
Fujimoto, S., Hoof, H., \BBA\ Meger, D. \BBOP2018\BBCP.
\newblock \BBOQ Addressing function approximation error in actor-critic methods\BBCQ\
\newblock In {\Bem International conference on machine learning}, \BPGS\ 1587--1596. PMLR.

\bibitem[\protect\BCAY{Fujimoto, Meger,\ \BBA\ Precup}{Fujimoto et~al.}{2019}]{BCQ}
Fujimoto, S., Meger, D., \BBA\ Precup, D. \BBOP2019\BBCP.
\newblock \BBOQ Off-policy deep reinforcement learning without exploration\BBCQ\
\newblock In {\Bem International conference on machine learning}, \BPGS\ 2052--2062. PMLR.

\bibitem[\protect\BCAY{Ghasemipour, Gu,\ \BBA\ Nachum}{Ghasemipour et~al.}{2022}]{msg}
Ghasemipour, K., Gu, S.~S., \BBA\ Nachum, O. \BBOP2022\BBCP.
\newblock \BBOQ Why so pessimistic? estimating uncertainties for offline {RL} through ensembles, and why their independence matters\BBCQ\
\newblock {\Bem Advances in Neural Information Processing Systems}, {\Bem 35}, 18267--18281.

\bibitem[\protect\BCAY{Haarnoja, Zhou, Abbeel,\ \BBA\ Levine}{Haarnoja et~al.}{2018}]{sac}
Haarnoja, T., Zhou, A., Abbeel, P., \BBA\ Levine, S. \BBOP2018\BBCP.
\newblock \BBOQ Soft actor-critic: Off-policy maximum entropy deep reinforcement learning with a stochastic actor\BBCQ\
\newblock In {\Bem International conference on machine learning}, \BPGS\ 1861--1870. PMLR.

\bibitem[\protect\BCAY{Hessel, Modayil, Van~Hasselt, Schaul, Ostrovski, Dabney, Horgan, Piot, Azar,\ \BBA\ Silver}{Hessel et~al.}{2018}]{rainbow}
Hessel, M., Modayil, J., Van~Hasselt, H., Schaul, T., Ostrovski, G., Dabney, W., Horgan, D., Piot, B., Azar, M., \BBA\ Silver, D. \BBOP2018\BBCP.
\newblock \BBOQ Rainbow: Combining improvements in deep reinforcement learning\BBCQ\
\newblock In {\Bem Proceedings of the AAAI conference on artificial intelligence}, \lowercase{\BVOL}~32.

\bibitem[\protect\BCAY{Jin, Yang, Wang,\ \BBA\ Jordan}{Jin et~al.}{2020}]{lsvi-2020}
Jin, C., Yang, Z., Wang, Z., \BBA\ Jordan, M.~I. \BBOP2020\BBCP.
\newblock \BBOQ Provably efficient reinforcement learning with linear function approximation\BBCQ\
\newblock In {\Bem Conference on Learning Theory}, \BPGS\ 2137--2143. PMLR.

\bibitem[\protect\BCAY{Jin, Yang,\ \BBA\ Wang}{Jin et~al.}{2021}]{jin2021pessimism}
Jin, Y., Yang, Z., \BBA\ Wang, Z. \BBOP2021\BBCP.
\newblock \BBOQ Is pessimism provably efficient for offline {RL}?\BBCQ\
\newblock In {\Bem International Conference on Machine Learning}, \BPGS\ 5084--5096. PMLR.

\bibitem[\protect\BCAY{Kiran, Sobh, Talpaert, Mannion, Al~Sallab, Yogamani,\ \BBA\ P{\'e}rez}{Kiran et~al.}{2021}]{autodriving}
Kiran, B.~R., Sobh, I., Talpaert, V., Mannion, P., Al~Sallab, A.~A., Yogamani, S., \BBA\ P{\'e}rez, P. \BBOP2021\BBCP.
\newblock \BBOQ Deep reinforcement learning for autonomous driving: A survey\BBCQ\
\newblock {\Bem IEEE Transactions on Intelligent Transportation Systems}, {\Bem 23\/}(6), 4909--4926.

\bibitem[\protect\BCAY{Kostrikov, Nair,\ \BBA\ Levine}{Kostrikov et~al.}{2022}]{iql}
Kostrikov, I., Nair, A., \BBA\ Levine, S. \BBOP2022\BBCP.
\newblock \BBOQ Offline reinforcement learning with implicit {Q}-learning\BBCQ\
\newblock In {\Bem ICLR}. OpenReview.net.

\bibitem[\protect\BCAY{Kumar, Zhou, Tucker,\ \BBA\ Levine}{Kumar et~al.}{2020}]{cql}
Kumar, A., Zhou, A., Tucker, G., \BBA\ Levine, S. \BBOP2020\BBCP.
\newblock \BBOQ Conservative {Q}-learning for offline reinforcement learning\BBCQ\
\newblock {\Bem Advances in Neural Information Processing Systems}, {\Bem 33}, 1179--1191.

\bibitem[\protect\BCAY{Lambert, Wulfmeier, Whitney, Byravan, Bloesch, Dasagi, Hertweck,\ \BBA\ Riedmiller}{Lambert et~al.}{2022}]{challenges}
Lambert, N., Wulfmeier, M., Whitney, W., Byravan, A., Bloesch, M., Dasagi, V., Hertweck, T., \BBA\ Riedmiller, M. \BBOP2022\BBCP.
\newblock \BBOQ The challenges of exploration for offline reinforcement learning\BBCQ\
\newblock {\Bem CoRR}, {\Bem abs/2201.11861}.

\bibitem[\protect\BCAY{Lan, Pan, Fyshe,\ \BBA\ White}{Lan et~al.}{2020}]{maxminQ}
Lan, Q., Pan, Y., Fyshe, A., \BBA\ White, M. \BBOP2020\BBCP.
\newblock \BBOQ Maxmin {Q}-learning: Controlling the estimation bias of {Q}-learning\BBCQ\
\newblock {\Bem CoRR}, {\Bem abs/2002.06487}.

\bibitem[\protect\BCAY{Lange, Gabel,\ \BBA\ Riedmiller}{Lange et~al.}{2012}]{BatchRL}
Lange, S., Gabel, T., \BBA\ Riedmiller, M. \BBOP2012\BBCP.
\newblock \BBOQ Batch reinforcement learning\BBCQ\
\newblock In {\Bem Reinforcement learning: State-of-the-art}, \BPGS\ 45--73. Springer.

\bibitem[\protect\BCAY{Lee, Laskin, Srinivas,\ \BBA\ Abbeel}{Lee et~al.}{2021}]{sunrise}
Lee, K., Laskin, M., Srinivas, A., \BBA\ Abbeel, P. \BBOP2021\BBCP.
\newblock \BBOQ Sunrise: A simple unified framework for ensemble learning in deep reinforcement learning\BBCQ\
\newblock In {\Bem International Conference on Machine Learning}, \BPGS\ 6131--6141. PMLR.

\bibitem[\protect\BCAY{Lee, Seo, Lee, Abbeel,\ \BBA\ Shin}{Lee et~al.}{2022}]{balancedreplay}
Lee, S., Seo, Y., Lee, K., Abbeel, P., \BBA\ Shin, J. \BBOP2022\BBCP.
\newblock \BBOQ Offline-to-online reinforcement learning via balanced replay and pessimistic {Q}-ensemble\BBCQ\
\newblock In {\Bem Conference on Robot Learning}, \BPGS\ 1702--1712. PMLR.

\bibitem[\protect\BCAY{Mnih, Kavukcuoglu, Silver, Rusu, Veness, Bellemare, Graves, Riedmiller, Fidjeland, Ostrovski, et~al.}{Mnih et~al.}{2015}]{natureDQN}
Mnih, V., Kavukcuoglu, K., Silver, D., Rusu, A.~A., Veness, J., Bellemare, M.~G., Graves, A., Riedmiller, M., Fidjeland, A.~K., Ostrovski, G., et~al. \BBOP2015\BBCP.
\newblock \BBOQ Human-level control through deep reinforcement learning\BBCQ\
\newblock {\Bem nature}, {\Bem 518\/}(7540), 529--533.

\bibitem[\protect\BCAY{Nair, Gupta, Dalal,\ \BBA\ Levine}{Nair et~al.}{2020}]{awac}
Nair, A., Gupta, A., Dalal, M., \BBA\ Levine, S. \BBOP2020\BBCP.
\newblock \BBOQ {AWAC}: Accelerating online reinforcement learning with offline datasets\BBCQ\
\newblock {\Bem CoRR}, {\Bem abs/2006.09359}.

\bibitem[\protect\BCAY{Nakamoto, Zhai, Singh, Mark, Ma, Finn, Kumar,\ \BBA\ Levine}{Nakamoto et~al.}{2023}]{cal-ql}
Nakamoto, M., Zhai, Y., Singh, A., Mark, M.~S., Ma, Y., Finn, C., Kumar, A., \BBA\ Levine, S. \BBOP2023\BBCP.
\newblock \BBOQ Cal-{QL}: Calibrated offline {RL} pre-training for efficient online fine-tuning\BBCQ\
\newblock {\Bem CoRR}, {\Bem abs/2303.05479}.

\bibitem[\protect\BCAY{Osband, Blundell, Pritzel,\ \BBA\ Van~Roy}{Osband et~al.}{2016}]{bootstrappedDQN}
Osband, I., Blundell, C., Pritzel, A., \BBA\ Van~Roy, B. \BBOP2016\BBCP.
\newblock \BBOQ Deep exploration via bootstrapped {DQN}\BBCQ\
\newblock {\Bem Advances in neural information processing systems}, {\Bem 29}.

\bibitem[\protect\BCAY{Raffin, Hill, Gleave, Kanervisto, Ernestus,\ \BBA\ Dormann}{Raffin et~al.}{2021}]{stablebaselines3}
Raffin, A., Hill, A., Gleave, A., Kanervisto, A., Ernestus, M., \BBA\ Dormann, N. \BBOP2021\BBCP.
\newblock \BBOQ Stable-baselines3: Reliable reinforcement learning implementations\BBCQ\
\newblock {\Bem Journal of Machine Learning Research}, {\Bem 22\/}(268), 1--8.

\bibitem[\protect\BCAY{Schneegass, Udluft,\ \BBA\ Martinetz}{Schneegass et~al.}{2008}]{uncertainty}
Schneegass, D., Udluft, S., \BBA\ Martinetz, T. \BBOP2008\BBCP.
\newblock \BBOQ Uncertainty propagation for quality assurance in reinforcement learning\BBCQ\
\newblock In {\Bem 2008 IEEE International Joint Conference on Neural Networks (IEEE World Congress on Computational Intelligence)}, \BPGS\ 2588--2595. IEEE.

\bibitem[\protect\BCAY{Schulman, Levine, Abbeel, Jordan,\ \BBA\ Moritz}{Schulman et~al.}{2015}]{TRPO}
Schulman, J., Levine, S., Abbeel, P., Jordan, M., \BBA\ Moritz, P. \BBOP2015\BBCP.
\newblock \BBOQ Trust region policy optimization\BBCQ\
\newblock In {\Bem International conference on machine learning}, \BPGS\ 1889--1897. PMLR.

\bibitem[\protect\BCAY{Schulman, Wolski, Dhariwal, Radford,\ \BBA\ Klimov}{Schulman et~al.}{2017}]{ppo}
Schulman, J., Wolski, F., Dhariwal, P., Radford, A., \BBA\ Klimov, O. \BBOP2017\BBCP.
\newblock \BBOQ Proximal policy optimization algorithms\BBCQ\
\newblock {\Bem CoRR}, {\Bem abs/1707.06347}.

\bibitem[\protect\BCAY{Schweighofer, Dinu, Radler, Hofmarcher, Patil, Bitto-Nemling, Eghbal-zadeh,\ \BBA\ Hochreiter}{Schweighofer et~al.}{2022}]{dataset}
Schweighofer, K., Dinu, M.-c., Radler, A., Hofmarcher, M., Patil, V.~P., Bitto-Nemling, A., Eghbal-zadeh, H., \BBA\ Hochreiter, S. \BBOP2022\BBCP.
\newblock \BBOQ A dataset perspective on offline reinforcement learning\BBCQ\
\newblock In {\Bem Conference on Lifelong Learning Agents}, \BPGS\ 470--517. PMLR.

\bibitem[\protect\BCAY{Shen, Li, Jiang, Wang,\ \BBA\ Zhao}{Shen et~al.}{2020}]{sr2l}
Shen, Q., Li, Y., Jiang, H., Wang, Z., \BBA\ Zhao, T. \BBOP2020\BBCP.
\newblock \BBOQ Deep reinforcement learning with robust and smooth policy\BBCQ\
\newblock In {\Bem International Conference on Machine Learning}, \BPGS\ 8707--8718. PMLR.

\bibitem[\protect\BCAY{Silver, Hubert, Schrittwieser, Antonoglou, Lai, Guez, Lanctot, Sifre, Kumaran, Graepel, et~al.}{Silver et~al.}{2018}]{alphazero}
Silver, D., Hubert, T., Schrittwieser, J., Antonoglou, I., Lai, M., Guez, A., Lanctot, M., Sifre, L., Kumaran, D., Graepel, T., et~al. \BBOP2018\BBCP.
\newblock \BBOQ A general reinforcement learning algorithm that masters chess, shogi, and go through self-play\BBCQ\
\newblock {\Bem Science}, {\Bem 362\/}(6419), 1140--1144.

\bibitem[\protect\BCAY{Sinha, Mandlekar,\ \BBA\ Garg}{Sinha et~al.}{2022}]{s4rl}
Sinha, S., Mandlekar, A., \BBA\ Garg, A. \BBOP2022\BBCP.
\newblock \BBOQ S4rl: Surprisingly simple self-supervision for offline reinforcement learning in robotics\BBCQ\
\newblock In {\Bem Conference on Robot Learning}, \BPGS\ 907--917. PMLR.

\bibitem[\protect\BCAY{Swazinna, Udluft,\ \BBA\ Runkler}{Swazinna et~al.}{2021}]{MOOSE}
Swazinna, P., Udluft, S., \BBA\ Runkler, T. \BBOP2021\BBCP.
\newblock \BBOQ Overcoming model bias for robust offline deep reinforcement learning\BBCQ\
\newblock {\Bem Engineering Applications of Artificial Intelligence}, {\Bem 104}, 104366.

\bibitem[\protect\BCAY{Tobin, Fong, Ray, Schneider, Zaremba,\ \BBA\ Abbeel}{Tobin et~al.}{2017}]{domainrandom}
Tobin, J., Fong, R., Ray, A., Schneider, J., Zaremba, W., \BBA\ Abbeel, P. \BBOP2017\BBCP.
\newblock \BBOQ Domain randomization for transferring deep neural networks from simulation to the real world\BBCQ\
\newblock In {\Bem 2017 IEEE/RSJ international conference on intelligent robots and systems (IROS)}, \BPGS\ 23--30. IEEE.

\bibitem[\protect\BCAY{Uchendu, Xiao, Lu, Zhu, Yan, Simon, Bennice, Fu, Ma, Jiao, et~al.}{Uchendu et~al.}{2023}]{jsrl}
Uchendu, I., Xiao, T., Lu, Y., Zhu, B., Yan, M., Simon, J., Bennice, M., Fu, C., Ma, C., Jiao, J., et~al. \BBOP2023\BBCP.
\newblock \BBOQ Jump-start reinforcement learning\BBCQ\
\newblock In {\Bem International Conference on Machine Learning}, \BPGS\ 34556--34583. PMLR.

\bibitem[\protect\BCAY{Van~der Maaten\ \BBA\ Hinton}{Van~der Maaten\ \BBA\ Hinton}{2008}]{tsne}
Van~der Maaten, L.\BBACOMMA\  \BBA\ Hinton, G. \BBOP2008\BBCP.
\newblock \BBOQ Visualizing data using {t-SNE}.\BBCQ\
\newblock {\Bem Journal of machine learning research}, {\Bem 9\/}(11).

\bibitem[\protect\BCAY{Wu, Wu, Qiu, Wang,\ \BBA\ Long}{Wu et~al.}{2022}]{SPOT}
Wu, J., Wu, H., Qiu, Z., Wang, J., \BBA\ Long, M. \BBOP2022\BBCP.
\newblock \BBOQ Supported policy optimization for offline reinforcement learning\BBCQ\
\newblock {\Bem Advances in Neural Information Processing Systems}, {\Bem 35}, 31278--31291.

\bibitem[\protect\BCAY{Wu, Tucker,\ \BBA\ Nachum}{Wu et~al.}{2019}]{BARC}
Wu, Y., Tucker, G., \BBA\ Nachum, O. \BBOP2019\BBCP.
\newblock \BBOQ Behavior regularized offline reinforcement learning\BBCQ\
\newblock {\Bem CoRR}, {\Bem abs/1911.11361}.

\bibitem[\protect\BCAY{Yang\ \BBA\ Wang}{Yang\ \BBA\ Wang}{2019}]{yang2019sample}
Yang, L.\BBACOMMA\  \BBA\ Wang, M. \BBOP2019\BBCP.
\newblock \BBOQ Sample-optimal parametric {Q}-learning using linearly additive features\BBCQ\
\newblock In {\Bem International Conference on Machine Learning}, \BPGS\ 6995--7004. PMLR.

\bibitem[\protect\BCAY{Yang, Bai, Ma, Wang, Zhang,\ \BBA\ Han}{Yang et~al.}{2022}]{rorl}
Yang, R., Bai, C., Ma, X., Wang, Z., Zhang, C., \BBA\ Han, L. \BBOP2022\BBCP.
\newblock \BBOQ {RORL}: Robust offline reinforcement learning via conservative smoothing\BBCQ\
\newblock {\Bem Advances in Neural Information Processing Systems}, {\Bem 35}, 23851--23866.

\bibitem[\protect\BCAY{Yu, Liu, Nemati,\ \BBA\ Yin}{Yu et~al.}{2021}]{healthcare}
Yu, C., Liu, J., Nemati, S., \BBA\ Yin, G. \BBOP2021\BBCP.
\newblock \BBOQ Reinforcement learning in healthcare: A survey\BBCQ\
\newblock {\Bem ACM Computing Surveys (CSUR)}, {\Bem 55\/}(1), 1--36.

\bibitem[\protect\BCAY{Yu, Thomas, Yu, Ermon, Zou, Levine, Finn,\ \BBA\ Ma}{Yu et~al.}{2020}]{mopo}
Yu, T., Thomas, G., Yu, L., Ermon, S., Zou, J.~Y., Levine, S., Finn, C., \BBA\ Ma, T. \BBOP2020\BBCP.
\newblock \BBOQ {MOPO}: Model-based offline policy optimization\BBCQ\
\newblock {\Bem Advances in Neural Information Processing Systems}, {\Bem 33}, 14129--14142.

\bibitem[\protect\BCAY{Zhang, Xu,\ \BBA\ Yu}{Zhang et~al.}{2023}]{pex}
Zhang, H., Xu, W., \BBA\ Yu, H. \BBOP2023\BBCP.
\newblock \BBOQ Policy expansion for bridging offline-to-online reinforcement learning\BBCQ\
\newblock In {\Bem The Eleventh International Conference on Learning Representations}.

\bibitem[\protect\BCAY{Zhao, Ma, Liu, Jianye, Zheng,\ \BBA\ Meng}{Zhao et~al.}{2023}]{e2o}
Zhao, K., Ma, Y., Liu, J., Jianye, H., Zheng, Y., \BBA\ Meng, Z. \BBOP2023\BBCP.
\newblock \BBOQ Improving offline-to-online reinforcement learning with {Q}-ensembles\BBCQ\
\newblock In {\Bem ICML Workshop on New Frontiers in Learning, Control, and Dynamical Systems}.

\bibitem[\protect\BCAY{Zhao, Boney, Ilin, Kannala,\ \BBA\ Pajarinen}{Zhao et~al.}{2022}]{adaptivebc}
Zhao, Y., Boney, R., Ilin, A., Kannala, J., \BBA\ Pajarinen, J. \BBOP2022\BBCP.
\newblock \BBOQ Adaptive behavior cloning regularization for stable offline-to-online reinforcement learning\BBCQ\
\newblock In {\Bem {ESANN}}.

\end{thebibliography}
\bibliographystyle{theapa}

\end{document}